\documentclass[letterpaper,11pt]{article}
\usepackage{graphicx}  %Required
\usepackage{fullpage}
\usepackage{amsfonts,amssymb,amsmath}
\usepackage{amsthm}
\usepackage{caption}
\usepackage{subcaption}
\usepackage[numbers]{natbib}
\usepackage{authblk}

\newtheorem{theorem}{Theorem}[section]
\newtheorem{lemma}[theorem]{Lemma}

\newtheorem{proposition}[theorem]{Proposition}

\theoremstyle{definition}
\newtheorem{definition}[theorem]{Definition}
\newtheorem{example}[theorem]{Example}

\newcommand{\SwD}{\triangleright}
\newcommand{\notSwD}{\ntriangleright}
\newcommand{\myvec}[1]{\boldsymbol{\mathbf{#1}}}
\def\X{\mathcal{X}}

\def\S{\mathcal{S}}

\usepackage{xcolor}
\usepackage{hyperref}
\hypersetup{
    colorlinks = true,
 	linkcolor = teal,
 	citecolor = teal,
	urlcolor = teal
}

\allowdisplaybreaks

\title{A Voting-Based System for Ethical Decision Making}

\author[1]{Ritesh Noothigattu}
\author[2]{Snehalkumar `Neil' S. Gaikwad}
\author[2]{Edmond Awad}
\author[2]{Sohan Dsouza}
\author[2]{Iyad Rahwan}
\author[1]{Pradeep Ravikumar}
\author[1]{Ariel D. Procaccia}
\affil[1]{School of Computer Science, Carnegie Mellon University}
\affil[2]{The Media Lab, Massachusetts Institute of Technology}

\date{}

\begin{document}

\maketitle

\begin{abstract}
We present a general approach to automating ethical decisions, drawing on machine learning and computational social choice. In a nutshell, we propose to \emph{learn} a model of societal preferences, and, when faced with a specific ethical dilemma at runtime, efficiently \emph{aggregate} those preferences to identify a desirable choice. We provide a concrete algorithm that instantiates our approach; some of its crucial steps are informed by a new theory of \emph{swap-dominance efficient} voting rules. Finally, we implement and evaluate a system for ethical decision making in the autonomous vehicle domain, using preference data collected from 1.3 million people through the Moral Machine website. 
\end{abstract}

\section{Introduction}\label{sec:intro}

The problem of ethical decision making, which has long been a grand challenge for AI~\cite{WA08}, has recently caught the public imagination. Perhaps its best-known manifestation is a modern variant of the classic \emph{trolley problem}~\cite{Jar85}: An autonomous vehicle has a brake failure, leading to an accident with inevitably tragic consequences; due to the vehicle's superior perception and computation capabilities, it can make an informed decision. Should it stay its course and hit a wall, killing its three passengers, one of whom is a young girl? Or swerve and kill a male athlete and his dog, who are crossing the street on a red light? A notable paper by \citet{BSR16} has shed some light on how people address such questions, and even former US President Barack Obama has weighed in.\footnote{{\fontfamily{cmvtt}\selectfont https://www.wired.com/2016/10/president-obama-mit-joi-ito-interview/}} 

Arguably the main obstacle to automating ethical decisions is the lack of a formal specification of ground-truth \emph{ethical principles}, which have been the subject of debate for centuries among ethicists and moral philosophers~\cite{Rawls71,Will86}.  In their work on fairness in machine learning, \citet{DHPR+12} concede that, when ground-truth ethical principles are not available, we must use an ``approximation as agreed upon by society.'' But how can society agree on the ground truth\,---\,or an approximation thereof\,---\,when even ethicists cannot? 

We submit that decision making can, in fact, be automated, even in the absence of such ground-truth principles, by aggregating people's opinions on ethical dilemmas. This view is foreshadowed by recent position papers by \citet{GRTV16} and \citet{CSSD+17}, who suggest that the field of \emph{computational social choice}~\cite{BCEL+16}, which deals with algorithms for aggregating individual preferences towards collective decisions, may provide tools for ethical decision making. In particular, Conitzer et al.~raise the possibility of ``letting our \emph{models} of multiple people's moral values \emph{vote} over the relevant alternatives.''

We take these ideas a step further by proposing and implementing a concrete approach for ethical decision making based on computational social choice, which, we believe, is quite practical. In addition to serving as a foundation for incorporating future ground-truth ethical and legal principles, it could even provide crucial preliminary guidance on some of the questions faced by ethicists. Our approach consists of four steps: 

\begin{enumerate}
\item[I] \emph{Data collection:} Ask human voters to compare pairs of alternatives (say a few dozen per voter). In the autonomous vehicle domain, an alternative is determined by a vector of features such as the number of victims and their gender, age, health\,---\,even species! 

\item[II] \emph{Learning:} Use the pairwise comparisons to learn a model of the preferences of each voter over all possible alternatives. 

\item[III] \emph{Summarization:} Combine the individual models into a single model, which approximately captures the collective preferences of all voters over all possible alternatives. 
 
\item[IV] \emph{Aggregation:} At runtime, when encountering an ethical dilemma involving a specific subset of alternatives, use the summary model to deduce the preferences of all voters over this particular subset, and apply a voting rule to aggregate these preferences into a collective decision. In the autonomous vehicle domain, the selected alternative is the outcome that society (as represented by the voters whose preferences were elicited in Step~I) views as the least catastrophic among the grim options the vehicle currently faces. Note that this step is only applied when all other options have been exhausted, i.e., all technical ways of avoiding the dilemma in the first place have failed, and all legal constraints that may dictate what to do have also failed.
\end{enumerate}

For Step I, we note that it is possible to collect an adequate dataset through, say, Amazon Mechanical Turk. But we actually perform this step on a much larger scale. Indeed, we use, for the first time, a unique dataset that consists of 18,254,285 pairwise comparisons between alternatives in the autonomous vehicle domain, obtained from 1,303,778 voters, through the website Moral Machine~\cite{ADKS+18}.\footnote{{\fontfamily{cmvtt}\selectfont http://moralmachine.mit.edu}}

Subsequent steps (namely Steps II, III, and IV) rely on having a \emph{model} for preferences. There is a considerable line of work on distributions over rankings over a \emph{finite} set of alternatives. A popular class of such models is the class of \emph{random utility models}, which use random utilities for alternatives to generate rankings over the alternatives. % Roughly speaking, in a random utility model, each alternative has a parameterized utility distribution (e.g., a Gaussian with specific mean and variance); to generate a ranking over the alternatives, a utility for each alternative is drawn, and the alternatives are sorted by utility. 
We require a slightly more general notion, as we are interested in situations where the set of alternatives is infinite, and any finite subset of alternatives might be encountered (c.f.~\cite{CT12}). For example, there are uncountably many scenarios an autonomous vehicle might face, because one can choose to model some features (such as the age of, say, a passenger) as continuous, but at runtime the vehicle will face a finite number of options. We refer to these generalized models as \emph{permutation processes}. 
 
In Section~\ref{sec:aggregation}, we focus on developing a theory of aggregation of permutation processes, which is crucial for Step~IV. Specifically, we assume that societal preferences are represented as a single permutation process. Given a (finite) subset of alternatives, the permutation process induces a distribution over rankings of these alternatives. In the spirit of \emph{distributional rank aggregation}~\cite{PPR15}, we view this distribution over rankings as an \emph{anonymous preference profile}, where the probability of a ranking is the fraction of voters whose preferences are represented by that ranking. This means we can apply a voting rule in order to aggregate the preferences\,---\,but \emph{which} voting rule should we apply? And how can we compute the outcome \emph{efficiently}? These are some of the central questions in computational social choice, but we show that in our context, under rather weak assumptions on the voting rule and permutation process, they are both moot, in the sense that it is easy to identify alternatives chosen by any ``reasonable'' voting rule. In slightly more detail, we define the notion of \emph{swap dominance} between alternatives in a preference profile, and show that if the permutation process satisfies a natural property with respect to swap dominance (standard permutation processes do), and the voting rule is \emph{swap-dominance efficient} (all common voting rules are), then any alternative that swap dominates all other alternatives is an acceptable outcome. 

Armed with these theoretical developments, our task can be reduced to: learning a permutation process for each voter (Step II); summarizing these individual processes into a single permutation process that satisfies the required swap-dominance property (Step III); and using any swap-dominance efficient voting rule, which is computationally efficient given the swap-dominance property (Step IV).

In Section~\ref{sec:instantiation}, we present a concrete algorithm that instantiates our approach, for a specific permutation process, namely the Thurstone-Mosteller (TM) Process~\cite{Thur27,Mos51}, and with a specific linear parametrization of its underlying utility process in terms of the alternative features. While these simple choices have been made to illustrate the framework, we note that, in principle, the framework can be instantiated with more general and complex permutation processes. 

Finally, in Section~\ref{sec:evaluation}, we implement and evaluate our algorithm. We first present simulation results from synthetic data that validate the accuracy of its learning and summarization components. More importantly, we implement our algorithm on the aforementioned Moral Machine dataset, and empirically evaluate the resultant system for choosing among alternatives in the autonomous vehicle domain. Taken together, these results suggest that our approach, and the algorithmic instantiation thereof, provide a computationally and statistically attractive method for ethical decision making.

\section{Related Work}

To our knowledge, the first to connect computational social choice and ethical decision making are \citet{GRTV16}. In their position paper, they raise the possibility of modeling ethical principles as the preferences of a `dummy' agent that is part of a larger system, and ask whether different formalisms should be used to model individual and collective ethical principles. They also note that there is work on collective decision making subject to feasibility constraints, but ethical principles are too complex to be simply specified as a set of feasibility constraints.

A more recent position paper about ethical decision making in AI, by \citet{CSSD+17}, discusses a number of different frameworks, and, in particular, touches upon game-theoretic models, social choice, and machine learning. They point out that ``aggregating the moral views of multiple humans (through a combination of machine learning and social-choice theoretic techniques) may result in a morally better system than that of any individual human, for example because idiosyncratic moral mistakes made by individual humans are washed out in the aggregate.'' Also relevant to our work is their discussion of the representation of dilemmas by their key moral features, for the purposes of applying machine learning algorithms. 

Our paper is most closely related to parallel work by \citet{FSSD+18}, who introduce a framework for prioritizing patients in kidney exchange. Specifically, they collected preferences over 8 simplified patient types from 289 workers on Amazon Mechanical Turk, and used them to learn societal weights for these eight types. Roughly speaking, the weights are such that if a random person was asked to compare two patient types, the probability she would prefer one to the other is proportional to its weight. These weights are then used to break ties among multiple outcomes that maximize the number of matched patients (ties are broken according to the sum of weights of matched patients). In contrast to our approach, there is no explicit preference aggregation, and voting does not take place. In addition, their approach is specific to kidney exchange. Arguably the main limitation of their approach is the use of weights that induce pairwise comparison probabilities as weights that represent societal benefit from matching a patient.\footnote{For example, if, all else being equal, a young patient is preferred to an old patient with a probability of 0.9, it does not mean that the societal value of the young patient is 9 times higher than that of the old patient.} Nonetheless, the work of \citeauthor{FSSD+18} serves as another compelling proof of concept (in a different domain), providing additional evidence that ethical decisions can be automated through computational social choice and machine learning.

Finally, recall that the massive dataset we use for Step I comes from the Moral Machine website; the conference version of our paper~\cite{NGAD+18} is the first publication to use this dataset. However, the original purpose of the website was to understand how \emph{people} make ethical decisions in the autonomous vehicle domain; the results of this experiment are presented in a recently published paper~\cite{ADKS+18}. The starting point of our work was the realization that the Moral Machine dataset can be used not just to understand people, but also to automate decisions.

\section{Preliminaries}
\label{sec:perm}

Let $\X$ denote a potentially infinite set of alternatives. Given a finite subset $A\subseteq \X$, we are interested in the set $\mathcal{S}_A$ of \emph{rankings} over $A$. Such a ranking $\sigma\in \mathcal{S}_A$ can be interpreted as mapping alternatives to their positions, i.e., $\sigma(a)$ is the position of $a\in A$ (smaller is more preferred). Let $a\succ_\sigma b$ denote that $a$ is preferred to $b$ in $\sigma$, that is, $\sigma(a)<\sigma(b)$. For $\sigma\in\S_A$ and $B\subseteq A$, let $\sigma|_B$ denote the ranking $\sigma$ restricted to $B$. And for a distribution $P$ over $\S_A$ and $B\subseteq A$, define $P|_B$ in the natural way to be the restriction of $P$ to $B$, i.e., for all $\sigma'\in \S_B$, 
$$
P|_B(\sigma')=\sum_{\sigma\in\S_A:\ \sigma|_B=\sigma'} P(\sigma).  
$$

A \emph{permutation process} $\left\{\Pi(A) \,:\,A \subseteq \X, |A|\in \mathbb{N} \right\}$ is a collection of distributions over $\S_A$ for every finite subset of alternatives $A$.  We say that a permutation process is \emph{consistent} if 
$\Pi(A)|_B = \Pi(B)$ for any finite subsets of alternatives $B\subseteq A\subseteq \X$. In other words, for a consistent permutation process $\Pi$, the distribution induced by $\Pi$ over rankings of the alternatives in $B$ is nothing but the distribution obtained by marginalizing out the extra alternatives $A\setminus B$ from the distribution induced by $\Pi$ over rankings of the alternatives in $A$. This definition of consistency is closely related to the Luce Choice Axiom~\cite{Luce59}. 

A simple adaptation of folklore results~\cite{Mar95} shows that any permutation process that is consistent has a natural interpretation in terms of utilities. Specifically (and somewhat informally, to avoid introducing notation that will not be used later), given any consistent permutation process $\Pi$ over a set of alternatives $\X$ (such that $|\X|\leq \aleph_1$), there exists a stochastic process $U$ (indexed by $\X$) such that for any $A=\{x_1,\ldots,x_m\} \subseteq \X$, the probability of drawing $\sigma\in \S_A$ from $\Pi(A)$ is equal to the probability that $\text{sort}(U_{x_1}, U_{x_2},\cdots,U_{x_m})=\sigma$, where (perhaps obviously) $\text{sort}(\cdot)$ sorts the utilities in non-increasing order. We can allow ties in utilities, as long as $\text{sort}(\cdot)$ is endowed with some tie-breaking scheme, e.g., ties are broken lexicographically, which we will assume in the sequel. We refer to the stochastic process corresponding to a consistent permutation process as its \emph{utility process}, since it is semantically meaningful to obtain a permutation via sorting by utility.

As examples of natural permutation processes, we adapt the definitions of two well-known \emph{random utility models}. The (relatively minor) difference is that random utility models define a distribution over rankings over a fixed, finite subset of alternatives, whereas permutation processes define a distribution for each finite subset of alternatives, given a potentially infinite space of alternatives.

\begin{itemize}
\item \textbf{Thurstone-Mosteller (TM) Process}~\cite{Thur27,Mos51}. A Thurstone-Mosteller Process (adaptation of Thurstone’s Case V model) is a consistent permutation process, whose utility process $U$ is a Gaussian process with independent utilities and identical variances. In more detail, given a finite set of alternatives $\{x_1, x_2, \cdots, x_m\}$, the utilities $(U_{x_1}, U_{x_2}, \cdots, U_{x_m})$ are independent, and $U_{x_i} \sim \mathcal{N}(\mu_{x_i}, \frac{1}{2})$, where $\mu_{x_i}$ denotes the mode utility of alternative $x_i$.

\item \textbf{Plackett-Luce (PL) Process}~\cite{Pla75,Luce59}. A Plackett-Luce Process is a consistent permutation process with the following utility process $U$: Given a finite set of alternatives $\{x_1, x_2, \cdots, x_m\}$, the utilities $(U_{x_1}, U_{x_2}, \cdots, U_{x_m})$ are independent, and each $U_{x_i}$ has a Gumbel distribution with identical scale, i.e. $U_{x_i} \sim \mathcal{G}(\mu_{x_i}, \gamma)$, where $\mathcal{G}$ denotes the Gumbel distribution, and $\mu_{x_i}$ denotes the mode utility of alternative $x_i$. We note that \citet{CT12} consider a further Bayesian extension of the above PL process, with a Gamma process prior over the mode utility parameters.
%\rncomment{I didn't also add direct permutation process interpretation of PL, since it would just take up more space}
\end{itemize}

%A PL is a consistent permutation process with the following utility process U: Given a finite subset of alternatives \{x_1, x_2, \cdots x_m\}, the utilities $U_{x_j}$ are independent and $U_{x_j}$ has a Gumbel distribution with mode $\mu_{x_j}$.

\section{Aggregation of Permutation Processes}
\label{sec:aggregation}

In social choice theory, a \emph{preference profile} is typically defined as a collection $\myvec{\sigma}=(\sigma_1,\ldots,\sigma_N)$ of $N$ rankings over a finite set of alternatives $A$, where $\sigma_i$ represents the preferences of voter $i$. However, when the identity of voters does not play a role, we can instead talk about an \emph{anonymous preference profile} $\pi \in [0,1]^{|A|!}$, where, for each $\sigma\in \mathcal{S}_A$, $\pi(\sigma)\in [0,1]$ is the \emph{fraction} of voters whose preferences are represented by the ranking $\sigma$. Equivalently, it is the probability that a voter drawn uniformly at random from the population has the ranking $\sigma$.

How is this related to permutation processes? Given a permutation process $\Pi$ and a finite subset $A\subseteq \X$, the distribution $\Pi(A)$ over rankings of $A$ can be seen as an anonymous preference profile $\pi$, where for $\sigma\in \mathcal{S}_A$, $\pi(\sigma)$ is the probability of $\sigma$ in $\Pi(A)$. As we shall see in Section~\ref{sec:instantiation}, Step II (learning) gives us a permutation process for each voter, where $\pi(\sigma)$ represents our \emph{confidence} that the preferences of the voter over $A$ coincide with $\sigma$; and after Step III (summarization), we obtain a single permutation process that represents societal preferences.

Our focus in this section is the aggregation of anonymous preference profiles induced by a permutation process (Step IV), that is, the task of choosing the winning alternative(s). To this end, let us define an \emph{anonymous social choice correspondence (SCC)} as a function $f$ that maps any anonymous preference profile $\pi$ over any finite and nonempty subset $A\subseteq \X$ to a nonempty subset of $A$. For example, under the ubiquitous \emph{plurality} correspondence, the set of selected alternatives consists of alternatives with maximum first-place votes, i.e., $\text{arg\,max}_{a\in A} \sum_{\sigma\in \mathcal{S}_A:\, \sigma(a)=1} \pi(\sigma)$; and under the \emph{Borda count} correspondence, denoting $|A|=m$, each vote awards $m-j$ points to the alternative ranked in position $j$, that is, the set of selected alternatives is $\text{arg\,max}_{a\in A} \sum_{j=1}^{m}(m-j)\sum_{\sigma\in \mathcal{S}_A:\, \sigma(a)=j} \pi(\sigma)$. We work with social choice \emph{correspondences} instead of social choice \emph{functions}, which return a single alternative in $A$, in order to smoothly handle ties.

\subsection{Efficient Aggregation}

Our main goal in this section is to address two related challenges. First, which (anonymous) social choice correspondence should we apply? There are many well-studied options, which satisfy different social choice axioms, and, in many cases, lead to completely different outcomes on the same preference profile. Second, how can we apply it in a computationally efficient way? This is not an easy task because, in general, we would need to explicitly construct the whole anonymous preference profile $\Pi(A)$, and then apply the SCC to it. The profile $\Pi(A)$ is of size $|A|!$, and hence this approach is intractable for a large $|A|$. Moreover, in some cases (such as the TM process), even computing the probability of a single ranking may be hard. The machinery we develop below allows us to completely circumvent these obstacles. 

Since stating our general main result requires some setup, we first state a simpler instantiation of the result for the specific TM and PL permutation processes (we will directly use this instantiation in Section~\ref{sec:instantiation}). Before doing so, we recall a few classic social choice axioms. We say that an anonymous SCC $f$ is \emph{monotonic} if the following conditions hold:
\begin{enumerate}
\item If $a\in f(\pi)$, and $\pi'$ is obtained by pushing $a$ upwards in the rankings, then $a\in f(\pi')$. 
%\item If $b\notin f(\pi)$, and $\pi'$ is obtained by pushing $b$ downwards in the rankings, then $b\notin f(\pi')$. 
\item If $a\in f(\pi)$ and $b\notin f(\pi)$, and $\pi'$ is obtained by pushing $a$ upwards in the rankings, then $b\notin f(\pi')$. 
\end{enumerate}
In addition, an anonymous SCC is \emph{neutral} if $f(\tau(\pi))=\tau(f(\pi))$ for any anonymous preference profile $\pi$, and any permutation $\tau$ on the alternatives; that is, the SCC is symmetric with respect to the alternatives (in the same way that anonymity can be interpreted as symmetry with respect to voters). 

\begin{theorem}
\label{thm:main}
Let $\Pi$ be the TM or PL process, let $A\subseteq \X$ be a nonempty, finite subset of alternatives, and let $a \in \text{arg\,max}_{x \in A}\ \mu_x$. Moreover, let $f$ be an anonymous SCC that is monotonic and neutral. Then $a\in f(\Pi(A))$.
\end{theorem}

To understand the implications of the theorem, we first note that many of the common voting rules, including plurality, Borda count (and, in fact, all positional scoring rules), Copeland, maximin, and Bucklin~\cite{BCEL+16}, are associated with anonymous, neutral, and monotonic SCCs. Specifically, all of these rules have a notion of score, and the SCC simply selects all the alternatives tied for the top score (typically there is only one).\footnote{Readers who are experts in social choice have probably noted that there are no social choice \emph{functions} that are both anonymous and neutral~\cite{Moul83}, intuitively because it is impossible to break ties in a neutral way. This is precisely why we work with social choice \emph{correspondences}.} The theorem then implies that all of these rules would agree that, given a subset of alternatives $A$, an alternative $a\in A$ with maximum mode utility is an acceptable winner, i.e., it is at least tied for the highest score, if it is not the unique winner. As we will see in Section~\ref{sec:instantiation}, such an alternative is very easy to identify, which is why, in our view, Theorem~\ref{thm:main} gives a satisfying solution to the challenges posed at the beginning of this subsection. We emphasize that this is merely an instantiation of Theorem~\ref{thm:apply-rule-proc}, which provides our result for general permutation processes.

The rest of this subsection is devoted to building the conceptual framework, and stating and proving the lemmas, required for the proof of Theorem~\ref{thm:main}, as well as to the statement and proof of Theorem~\ref{thm:apply-rule-proc}. 

Starting off, let $\pi$ denote an anonymous preference profile (or distribution over rankings) over alternatives $A$. We define the ranking $\sigma^{ab}$ as the ranking $\sigma$ with alternatives $a$ and $b$ swapped, i.e. $\sigma^{ab}(x)=\sigma(x)$ if $x\in A\setminus \{a,b\}$, $\sigma^{ab}(b)=\sigma(a)$, and $\sigma^{ab}(a)=\sigma(b)$.
%$$\sigma^{ab}(x) = \begin{cases} \sigma(x), & \text{if } x \neq a,b \\ \sigma(b), & \text{if } x = a \\ \sigma(a), & \text{if } x = b \end{cases}$$
\begin{definition}
We say that alternative $a\in A$ \emph{swap-dominates} alternative $b\in A$ in anonymous preference profile $\pi$ over $A$\,---\,denoted by $a \SwD_\pi b$\,---\,if for every ranking $\sigma\in \mathcal{S}_A$ with $a \succ_\sigma b$ it holds that
$\pi(\sigma) \geq \pi(\sigma^{ab})$.
\end{definition}
In words, $a$ swap-dominates $b$ if every ranking that places $a$ above $b$ has at least as much weight as the ranking obtained by swapping the positions of $a$ and $b$, and keeping everything else fixed. This is a very strong dominance relation, and, in particular, implies existing dominance notions such as \emph{position dominance}~\cite{CPS16}.
%
%\begin{proposition}
%The relation $\SwD_\pi$ is always reflexive, but, may not be symmetric, transitive or total in general.
%\end{proposition}
%
%\begin{proof}
%$\SwD_\pi$ is trivially reflexive, because, for any $a \in A$, there is no ranking $\sigma$ with $a \succ_\sigma a$.
%
%For the other properties, consider three alternative $A = \{a,b,c\}$. Let $\pi$ be the profile with equal weight on the rankings: $(a \succ b \succ c)$, $(b \succ a \succ c)$ and $(c \succ a \succ b)$ (and no weight on the remaining rankings). In this profile, we have $a \SwD_\pi b$, but, $b \notSwD_\pi a$. Hence, the relation is not symmetric. Also, we have $a \SwD_\pi b$ and $b \SwD_\pi c$, but, $a \notSwD_\pi c$. Hence, the relation is also not transitive. Finally, we have $a \notSwD_\pi c$ and $c \notSwD_\pi a$. Hence, the relation is also not total.
%\end{proof}
%
Next we define a property of social choice correspondences, which intuitively requires that the correspondence adhere to swap dominance relations, if they exist in a given anonymous preference profile. 

\begin{definition} \label{def:swap-dom-eff}
An anonymous SCC $f$ is said to be \emph{swap-dominance-efficient (SwD-efficient)} if for every anonymous preference profile $\pi$ and any two alternatives $a$ and $b$, if $a$ swap-dominates $b$ in $\pi$, then $b\in f(\pi)$ implies $a\in f(\pi)$. 
\end{definition}

Because swap-dominance is such a strong dominance relation, SwD-efficiency is a very weak requirement, which is intuitively satisfied by almost any ``reasonable'' voting rule. This intuition is formalized in the following lemma. 
\begin{lemma} \label{lem:swap-dom-eff}
Any anonymous SCC that satisfies monotonicity and neutrality is SwD-efficient.
\end{lemma}

\begin{proof}
Let $f$ be an anonymous SCC that satisfies monotonicity and neutrality. Let $\pi$ be an arbitrary anonymous preference profile, and let $a, b$ be two arbitrary alternatives such that $a \SwD_\pi b$. Now, suppose for the sake of contradiction that $b\in f(\pi)$ but $a\notin f(\pi)$.

Consider an arbitrary ranking $\sigma$ with $a \succ_\sigma b$. Since $a \SwD_\pi b$, $\pi(\sigma) \geq \pi(\sigma^{ab})$. In other words, we have an excess weight of $\pi(\sigma) - \pi(\sigma^{ab})$ on $\sigma$. For this excess weight of $\sigma$, move $b$ upwards and place it just below $a$. By monotonicity, $b$ still wins and $a$ still loses in this modified profile. We repeat this procedure for every such $\sigma$ (i.e. for its excess weight, move $b$ upwards, until it is placed below $a$). In the resulting profile, $a$ still loses. Now, for each of the modified rankings, move $a$ down to where $b$ originally was. By monotonicity, $a$ still loses in the resulting profile $\pi'$, i.e., $a\notin f(\pi')$.

On the other hand, this procedure is equivalent to shifting the excess weight $\pi(\sigma) - \pi(\sigma^{ab})$ from $\sigma$ to $\sigma^{ab}$ (for each $\sigma$ with $a \succ_\sigma b$). Hence, the profile $\pi'$ we end up with is such that $\pi'(\sigma) = \pi(\sigma^{ab})$ and $\pi'(\sigma^{ab}) = \pi(\sigma)$, i.e. the new profile is the original profile with $a$ and $b$ swapped. Therefore, by neutrality, it must be the case that $a\in f(\pi')$. This contradicts our conclusion that $a\notin f(\pi')$, thus completing the proof.
\end{proof}

So far, we have defined a property, SwD-efficiency, that any SCC might potentially satisfy. But why is this useful in the context of aggregating permutation processes? We answer this question in Theorem~\ref{thm:apply-rule-proc}, but before stating it, we need to introduce the definition of a property that a \emph{permutation process} might satisfy.

\begin{definition}
Alternative $a \in \X$ swap-dominates alternative $b \in \X$ in the permutation process $\Pi$\,---\,denoted by $a \SwD_\Pi b$\,---\,if for every finite set of alternatives $A \subseteq \X$ such that $\{a,b\}\subseteq A$, $a$ swap-dominates $b$ in the anonymous preference profile $\Pi(A)$.
\end{definition}
%This might seem like a pretty strict property, but, commonly used permutation processes like the Thurstone-Mosteller process and Plackett-Luce process have tons of swap-dominance relations (specifically, they have swap-dominance relations between every pair of alternatives).
%
We recall that a \emph{total preorder} is a binary relation that is transitive and total (and therefore reflexive). 
\begin{definition} \label{def:total-order-proc}
A permutation process $\Pi$ over $\X$ is said to be \emph{SwD-compatible} if the binary relation $\SwD_\Pi$ is a total preorder on $\X$.
\end{definition}
We are now ready to state our main theorem.
\begin{theorem} \label{thm:apply-rule-proc}
Let $f$ be an SwD-efficient anonymous SCC, and let $\Pi$ be an SwD-compatible permutation process. Then for any finite subset of alternatives $A$, there exists $a\in A$ such that $a \SwD_\Pi b $ for all $b \in A$. Moreover, $a \in f(\Pi(A))$. 
\end{theorem}

\begin{proof}
Let $f$, $\Pi$, and $A$ as in the theorem statement. Since $\Pi$ is SwD-compatible, $\SwD_\Pi$ is a total preorder on $\X$. In turn, the relation $\SwD_\Pi$ restricted to $A$ is a total preorder on $A$. Therefore, there is $a \in A$ such that $a \SwD_\Pi b$ for all $b \in A$.

Suppose for the sake of contradiction that $a\notin f(\Pi(A))$, and let $b\in A\setminus \{a\}$. Then it holds that $a \SwD_\Pi b$. In particular, $a \SwD_{\Pi(A)} b$. But, because $f$ is SwD-efficient and $a\notin f(\Pi(A))$, we have that $b\notin f(\Pi(A))$. This is true for every $b \in A$, leading to $f(\Pi(A)) = \phi$, which contradicts the definition of an SCC.
\end{proof}

Theorem~\ref{thm:apply-rule-proc} asserts that for any SwD-compatible permutation process, any SwD-efficient SCC (which, as noted above, include most natural SCCs, namely those that are monotonic and neutral), given any finite set of alternatives, will always select a very natural winner that swap-dominates other alternatives. A practical use of this theorem requires two things: to show that the permutation process is SwD-compatible, and that it is computationally tractable to select an alternative that swap-dominates other alternatives in a finite subset. The next few lemmas provide some general recipes for establishing these properties for general permutation processes, and, in particular, we show that they indeed hold under the TM and PL processes. First, we have the following definition. 
\begin{definition} \label{def:dom-util-proc}
Alternative $a\in \X$ \emph{dominates} alternative $b\in \X$ in utility process $U$ if for every finite subset of alternatives containing $a$ and $b$, $\{a,b,x_3, \hdots x_m\} \subseteq \X$, and every vector of utilities $(u_1, u_2, u_3 \hdots u_m) \in \mathbb{R}^m$ with $u_1 \geq u_2$, it holds that
\begin{equation} \label{eqn:util-dom}
p_{(U_a, U_b, U_{x_3}, \hdots U_{x_m})}(u_1, u_2, u_3 \hdots u_m)
 \geq p_{(U_a, U_b, U_{x_3}, \hdots U_{x_m})}(u_2, u_1, u_3 \hdots u_m),
\end{equation}
where $p_{(U_a, U_b, U_{x_3}, \hdots U_{x_m})}$ is the density function of the random vector $(U_a, U_b, U_{x_3}, \hdots U_{x_m})$.
\end{definition}

Building on this definition, Lemmas~\ref{lem:dom-implies-SwD} and \ref{lem:TM-PL-dom} directly imply that the TM and PL processes are SwD-compatible.

\begin{lemma}\label{lem:dom-implies-SwD}
Let $\Pi$ be a consistent permutation process, and let $U$ be its corresponding utility process. If alternative $a$ dominates alternative $b$ in $U$, then $a$ swap-dominates $b$ in $\Pi$.
\end{lemma} 

\begin{proof}
Let $a$ and $b$ be two alternatives such that $a$ dominates $b$ in $U$. In addition, let $A$ be a finite set of alternatives containing $a$ and $b$, let $\pi$ denote the anonymous preference profile $\Pi(A)$, and let $m = |A|$. Consider an arbitrary ranking $\sigma$ such that $a \succ_\sigma b$. Now, let $x_\ell=\sigma^{-1}(\ell)$ denote the alternative in position $\ell$ of $\sigma$, and let $i = \sigma(a)$, $j = \sigma(b)$, i.e., $$x_1 \succ_\sigma x_2 \cdots \succ_\sigma x_i (= a)\succ_\sigma \cdots \succ_\sigma x_j (= b)\succ_\sigma \cdots \succ_\sigma x_m.$$ 
Then,
\small
\begin{align*}
\pi(\sigma)&= P(U_{x_1} > U_{x_2} > \cdots  >U_{x_i} > \cdots >  U_{x_j} > \cdots > U_{x_m})\\
&= \int_{-\infty}^\infty \int_{-\infty}^{u_1} \cdots \int_{-\infty}^{u_{i-1}} \cdots \int_{-\infty}^{u_{j-1}} \cdots \int_{-\infty}^{u_{m-1}} p(u_1, u_2, \cdots, u_i, \cdots u_j, \cdots, u_m) du_m \cdots du_1.
\end{align*}
\normalsize
In this integral, because of the limits, we always have $u_i \geq u_j$. Moreover, since $x_i = a$ dominates $x_j = b$ in $U$, we have
\begin{align*}
\pi(\sigma) &\geq \int_{-\infty}^\infty \int_{-\infty}^{u_1} \cdots \int_{-\infty}^{u_{i-1}} \cdots \int_{-\infty}^{u_{j-1}} \cdots \int_{-\infty}^{u_{m-1}} p(u_1, u_2, \cdots, u_j, \cdots u_i, \cdots, u_m) du_m \cdots du_1.
\end{align*}
The right-hand side of this equation is exactly $\pi(\sigma^{ab})$. Hence, we have $\pi(\sigma) \geq \pi(\sigma^{ab})$. It follows that $a \SwD_\pi b$, i.e., $a \SwD_{\Pi(A)} b$. Also, this is true for any finite $A$ containing $a$ and $b$. We conclude that $a \SwD_\Pi b$.
\end{proof}

\begin{lemma}\label{lem:TM-PL-dom}
Under the TM and PL processes, alternative $a$ dominates alternative $b$ in the corresponding utility process if and only if $\mu_a \geq \mu_b$.
\end{lemma}

\begin{proof}
We establish the property separately for the TM and PL processes.

\smallskip
\emph{TM process.} Let $a$ and $b$ be two alternatives such that $\mu_a \geq \mu_b$. Since we are dealing with a TM process, $U_a \sim \mathcal{N}(\mu_a, \frac{1}{2})$ and $U_b \sim \mathcal{N}(\mu_b, \frac{1}{2})$. Let $A$ be any finite set of alternatives containing $a$ and $b$. Since utilities are sampled independently in a TM process, the difference between the two sides of Equation~\eqref{eqn:util-dom} is that the left-hand side has $p_{U_a}(u_1) p_{U_b}(u_2)$, while the right-hand side has $p_{U_a}(u_2) p_{U_b}(u_1)$. It holds that
\begin{equation}
\label{eq:exp}
\begin{split}
&p_{U_a}(u_1) p_{U_b}(u_2)\\
 &\quad= \frac{1}{\sqrt{\pi}} \exp\left(-(u_1 - \mu_a)^2\right) \frac{1}{\sqrt{\pi}} \exp\left(-(u_2 - \mu_b)^2\right).\\
&\quad= \frac{1}{\pi}\exp\left(-u_1^2 - \mu_a^2 - u_2^2 - \mu_b^2 + 2u_1\mu_a + 2u_2\mu_b\right).
\end{split}
\end{equation}
We have $u_1 \geq u_2$ and $\mu_a \geq \mu_b$. Therefore, 
\begin{align*}
u_1\mu_a + u_2 \mu_b &= u_1\mu_b + u_1(\mu_a - \mu_b) + u_2 \mu_b\\
&\geq u_1 \mu_b + u_2(\mu_a - \mu_b) + u_2 \mu_b\\
&= u_1 \mu_b + u_2 \mu_a
\end{align*}
Substituting this into Equation~\eqref{eq:exp}, we obtain
\begin{align*}
p_{U_a}(u_1) p_{U_b}(u_2)&\geq \frac{1}{\pi}\exp\left(-u_1^2 - \mu_a^2 - u_2^2 - \mu_b^2 + 2u_1\mu_b + 2u_2\mu_a\right)\\
&= \frac{1}{\pi} \exp\left(- (u_2 - \mu_a)^2-(u_1 - \mu_b)^2 \right)\\
&= p_{U_a}(u_2) p_{U_b}(u_1)
\end{align*}
It follows that Equation~\eqref{eqn:util-dom} holds true. Hence, $a$ dominates $b$ in the corresponding utility process.

To show the other direction, let $a$ and $b$ be such that $\mu_a < \mu_b$. If we choose $u_1, u_2$ such that $u_1 > u_2$, using a very similar approach as above, we get $p_{U_a}(u_1) p_{U_b}(u_2) < p_{U_a}(u_2) p_{U_b}(u_1)$. And so, $a$ does not dominate $b$ in the corresponding utility process.\qed

\smallskip
\emph{PL process.} Let $a$ and $b$ be two alternatives such that $\mu_a \geq \mu_b$. Since we are dealing with a PL process, $U_a \sim \mathcal{G}(\mu_a, \gamma)$ and $U_b \sim \mathcal{G}(\mu_b, \gamma)$. Let $A$ be any finite set of alternatives containing $a$ and $b$. Since utilities are sampled independently in a PL process, the difference between the two sides of Equation~\eqref{eqn:util-dom} is that the left-hand side has $p_{U_a}(u_1) p_{U_b}(u_2)$, while the right-hand side has $p_{U_a}(u_2) p_{U_b}(u_1)$. It holds that
\begin{equation}
\label{eq:gumbel}
\begin{split}
p_{U_a}(u_1) p_{U_b}(u_2)
&= \frac{1}{\gamma} \exp\left( -\frac{u_1 - \mu_a}{\gamma} - e^{-\frac{u_1-\mu_a}{\gamma}} \right) \frac{1}{\gamma} \exp\left( -\frac{u_2 - \mu_b}{\gamma} - e^{-\frac{u_2-\mu_b}{\gamma}} \right)\\
&= \frac{1}{\gamma^2} \exp\left( -\frac{u_1 - \mu_a}{\gamma} - e^{-\frac{u_1-\mu_a}{\gamma}} -\frac{u_2 - \mu_b}{\gamma} - e^{-\frac{u_2-\mu_b}{\gamma}} \right)\\
&= \frac{1}{\gamma^2} \exp\left( -\frac{u_1 - \mu_a + u_2 - \mu_b}{\gamma} - \left(e^{-\frac{u_1}{\gamma}}e^{\frac{\mu_a}{\gamma}} + e^{-\frac{u_2}{\gamma}}e^{\frac{\mu_b}{\gamma}}\right) \right).
\end{split}
\end{equation}
We also know that $e^{-\frac{u_2}{\gamma}} \geq e^{-\frac{u_1}{\gamma}}$ and $e^{\frac{\mu_a}{\gamma}} \geq e^{\frac{\mu_b}{\gamma}}$. Similar to the proof for the TM process, we have
$$
e^{-\frac{u_2}{\gamma}}e^{\frac{\mu_a}{\gamma}} + e^{-\frac{u_1}{\gamma}}e^{\frac{\mu_b}{\gamma}} \geq e^{-\frac{u_1}{\gamma}}e^{\frac{\mu_a}{\gamma}} + e^{-\frac{u_2}{\gamma}}e^{\frac{\mu_b}{\gamma}}.
$$
Substituting this into Equation~\eqref{eq:gumbel}, we obtain
\begin{align*}
p_{U_a}(u_1) p_{U_b}(u_2)
 &\geq \frac{1}{\gamma^2} \exp\left( -\frac{u_1 - \mu_a + u_2 - \mu_b}{\gamma} - \left(e^{-\frac{u_2}{\gamma}}e^{\frac{\mu_a}{\gamma}} + e^{-\frac{u_1}{\gamma}}e^{\frac{\mu_b}{\gamma}}\right) \right)\\
&= \frac{1}{\gamma} \exp\left( -\frac{u_2 - \mu_a}{\gamma} - e^{-\frac{u_2-\mu_a}{\gamma}} \right) \frac{1}{\gamma} \exp\left( -\frac{u_1 - \mu_b}{\gamma} - e^{-\frac{u_1-\mu_b}{\gamma}} \right)\\
& = p_{U_a}(u_2) p_{U_b}(u_1)
\end{align*}
It follows that Equation~\eqref{eqn:util-dom} holds true. Hence, $a$ dominates $b$ in the corresponding utility process.

To show the other direction, let $a$ and $b$ be such that $\mu_a < \mu_b$. If we choose $u_1, u_2$ such that $u_1 > u_2$, using a very similar approach as above, we get $p_{U_a}(u_1) p_{U_b}(u_2) < p_{U_a}(u_2) p_{U_b}(u_1)$. And so, $a$ does not dominate $b$ in the corresponding utility process.
\end{proof}

The proof of Theorem~\ref{thm:main} now follows directly. 

\begin{proof}[Proof of Theorem~\ref{thm:main}]
By Lemma~\ref{lem:swap-dom-eff}, the anonymous SCC $f$ is SwD-efficient. Lemmas~\ref{lem:dom-implies-SwD} and \ref{lem:TM-PL-dom} directly imply that when $\Pi$ is the TM or PL process, $\SwD_\Pi$ is indeed a total preorder. In particular, $a \SwD_\Pi b$ if $\mu_a \geq \mu_b$. So, an alternative $a$ in $A$ with maximum mode utility satisfies $a\SwD_\Pi b$ for all $b\in A$. By Theorem~\ref{thm:apply-rule-proc}, if $a\in A$ is such that $a\SwD_\Pi b$ for all $b\in A$, then $a\in f(\Pi(A))$; the statement of the theorem follows.
\end{proof}

\subsection{Stability}

It turns out that the machinery developed for the proof of Theorem~\ref{thm:main} can be leveraged to establish an additional desirable property. 
\begin{definition} Given an anonymous SCC $f$, and a permutation process $\Pi$ over $\X$, we say that the pair $(\Pi, f)$ is \emph{stable} if for any nonempty and finite subset of alternatives $A \subseteq \X$, and any nonempty subset $B \subseteq A$,
$f(\Pi(A)) \cap B = f(\Pi(B))$ whenever $f(\Pi(A)) \cap B \neq \phi$.
\end{definition}
Intuitively, stability means that applying $f$ under the assumption that the set of alternatives is $A$, and then reducing to its subset $B$, is the same as directly reducing to $B$ and then applying $f$. This notion is related to classic axioms studied by \citet{Sen71}, specifically his \emph{expansion} and \emph{contraction} properties. In our setting, stability seems especially desirable, as our algorithm would potentially face decisions over many different subsets of alternatives, and the absence of stability may lead to glaringly inconsistent choices.

Our main result regarding stability is the following theorem.

\begin{theorem}
\label{thm:stability}
Let $\Pi$ be the TM or PL process, and let $f$ be the Borda count or Copeland SCC. Then the pair $(\Pi,f)$ is stable. 
\end{theorem}

The \emph{Copeland} SCC, which appears in the theorem statement, is defined as follows. For an anonymous preference profile $\pi$ over $A$, we say that $a\in A$ beats $b\in A$ in a pairwise election if $$\sum_{\sigma\in \mathcal{S}_A:\, a\succ_\sigma b} \pi(\sigma) > \frac 1 2.$$ The \emph{Copeland score} of an alternative is the number of other alternatives it beats in pairwise elections; the Copeland SCC selects all alternatives that maximize the Copeland score. 

The rest of the section is devoted to building intuition for, and proving, Theorem~\ref{thm:stability}. Among other things, the proof requires a stronger notion of SwD-efficiency, which, as we show, is satisfied by Borda and Copeland. We will then be able to derive Theorem~\ref{thm:stability} as a corollary of the more general Theorem~\ref{thm:total-swap-self-cons}. We start by examining some examples that illustrate stability (or the lack thereof).

%\begin{proposition} \label{prop:single-self-cons}
%Suppose a consistent permutation process $\Pi$ has all its mass on a single ranking (for any finite subset of alternatives). Then for any anonymous SCC $f$ that satisfies weak-unanimity, the tuple $(\Pi, f)$ satisfies stability.
%\end{proposition}
%
%\begin{proof}
%Let $A$ be an arbitrary finite subset of alternatives and let $B \subseteq A$. Now, let $\sigma$ be the ranking on which $\Pi(A)$ has all its mass. Let $a$ be the alternative at the top of this ranking, i.e. $a = \sigma^{-1}(1)$. Then, because $f$ satisfies weak-unanimity, $f(\Pi(A)) = \{a\}$. 
%
%Now, if $a \notin B$, then $f(\Pi(A)) \cap B = \phi$ and we are done. So, let's consider the case when $a \in B$. Since $\Pi$ is a consistent permutation process, the anonymous preference profile $\Pi(B)$ will has all its mass on the single ranking $\sigma|_B$. But, $a$ is the alternative at the top of $\sigma$ and also belongs to $B$, hence, $a$ is at the top of $\sigma|_B$ too. So, by weak-unanimity of $f$, $f(\Pi(B)) = \{a\}$. Therefore, $f(\Pi(A)) \cap B = f(\Pi(B)).$
%\end{proof}
%
%This proposition states that if a consistent permutation process has all its mass on a single ranking, then for most SCCs the process SCC tuple will be stable. Nonetheless, we can find simple examples where the tuple $(\Pi, f)$ does not satisfy stability.

\begin{example}	\label{ex:borda-self-cons}
Let $f$ be the Borda count SCC, and let the set of alternatives be $\X = \{u, v, w, x, y\}$. Also, let $\Pi$ be a consistent permutation process, which,  given all the alternatives, gives a uniform distribution on the two rankings $(x \succ u \succ v \succ y \succ w)$ and $(y \succ w \succ x \succ u \succ v)$. The outcome of applying $f$ on this profile is $\{x\}$ (since $x$ has the strictly highest Borda score). But, the outcome of applying $f$ on the profile $\Pi(\{w,x,y\})$ is $\{y\}$ (since $y$ now has the strictly highest Borda score). Hence, $f(\Pi(\{u,v,w,x,y\})) \cap \{w,x,y\} \neq f(\Pi(w,x,y))$, even though the left-hand side is nonempty. We conclude that the tuple $(\Pi, f)$ does not satisfy stability.
\end{example}

\begin{example}
Consider the permutation process of Example~\ref{ex:borda-self-cons}, and let $f$ be the Copeland SCC. Once again, it holds that $f(\Pi(u,v,w,x,y))=\{x\}$ and $f(\Pi(w,x,y))=\{y\}$. Hence the pair $(\Pi, f)$ is not stable.
\end{example}

Now, in the spirit of Theorem~\ref{thm:apply-rule-proc}, let us see whether the pair $(\Pi, f)$ satisfies stability when $f$ is an SwD-efficient anonymous SCC, and $\Pi$ is an SwD-compatible permutation process. Example~\ref{ex:plurality} constructs such a $\Pi$ that is not stable with respect to the plurality SCC (even though plurality is SwD-efficient). 

\begin{example} \label{ex:plurality}
Let $f$ be the plurality SCC and the set of alternatives be $\X = \{a,b,c\}$. Also, let $\Pi$ be the consistent permutation process, which given all alternatives, gives the following profile: $0.35$ weight on $(a \succ b \succ c)$, $0.35$ weight on $(b \succ a \succ c)$, $0.1$ weight on $(c \succ a \succ b)$, $0.1$ weight on $(a \succ c \succ b)$ and $0.1$ weight on $(b \succ c \succ a)$. All the swap-dominance relations in this permutation process are: $a \SwD_\Pi b$, $b \SwD_\Pi c$ and $a \SwD_\Pi c$. Hence, $\SwD_\Pi$ is a total preorder on $\X$, and $\Pi$ is SwD-compatible.\\
Now, for this permutation process $\Pi$ and the plurality SCC $f$, we have: $f(\Pi(\{a,b,c\})) = \{a,b\}$ and $f(\Pi(\{a,b\})) = \{a\}$. Therefore, $(\Pi, f)$ is not stable.
\end{example}

This happens because Plurality is not \textit{strongly} SwD-efficient, as defined below (Example~\ref{ex:plurality} even shows why plurality violates this property).

\begin{definition} \label{def:strong-swd}
An anonymous SCC $f$ is said to be \textit{strongly SwD-efficient} if for every anonymous preference profile $\pi$ over $A$, and any two alternatives $a,b\in A$ such that $a \SwD_\pi b$,
\begin{enumerate}
\item If $b \notSwD_\pi a$, then $b \notin f(\pi)$.
\item If $b \SwD_\pi a$, then $b \in f(\pi) \Leftrightarrow a \in f(\pi)$.
\end{enumerate}
\end{definition}

It is clear that any strongly SwD-efficient SCC is also SwD-efficient. 

\begin{lemma} \label{lem:comm-strong-SwD}
The Borda count and Copeland SCCs are strongly SwD-efficient.
\end{lemma}

\begin{proof}
Let $\pi$ be an arbitrary anonymous preference profile over alternatives $A$, and let $a, b \in A$ such that $a \SwD_\pi b$. This means that for all $\sigma \in \mathcal{S}_A$ with $a \succ_\sigma b$, we have $\pi(\sigma) \geq \pi(\sigma^{ab})$. We will examine the two conditions (of Definition~\ref{def:strong-swd}) separately.

\medskip
\noindent\textbf{Case 1:} $b \notSwD_\pi a$. This means that there exists a ranking $\sigma_* \in \mathcal{S}_A$ with $b \succ_{\sigma_*} a$ such that $\pi(\sigma_*) < \pi(\sigma_*^{ab})$. Below we analyze each of the SCCs mentioned in the theorem.

\emph{Borda count.} $\mathcal{S}_A$ can be partitioned into pairs of the form $(\sigma, \sigma^{ab})$, where $\sigma$ is such that $a \succ_\sigma b$. We reason about how each pair contributes to the Borda scores of $a$ and $b$. Consider an arbitrary pair $(\sigma, \sigma^{ab})$ with $a \succ_\sigma b$. The score contributed by $\sigma$ to $a$ is $(m - \sigma(a)) \pi(\sigma)$, and the score contributed to $b$ is $(m - \sigma(b)) \pi(\sigma)$. That is, it gives an excess score of $(\sigma(b) - \sigma(a)) \pi(\sigma)$ to $a$. Similarly, the score of $a$ contributed by $\sigma^{ab}$ is $(m - \sigma^{ab}(a))\pi(\sigma^{ab}) = (m - \sigma(b))\pi(\sigma^{ab})$, and the score contributed to $b$ is $(m - \sigma(a))\pi(\sigma^{ab})$. So, $b$ gets an excess score of $(\sigma(b) - \sigma(a))\pi(\sigma^{ab})$ from $\sigma^{ab}$. Combining these observations, the pair $(\sigma, \sigma^{ab})$ gives $a$ an excess score of $(\sigma(b) - \sigma(a))(\pi(\sigma) - \pi(\sigma^{ab}))$, which is at least $0$. Since this is true for every pair $(\sigma, \sigma^{ab})$, $a$ has Borda score that is at least as high as that of $b$. Furthermore, the pair $(\sigma_*^{ab}, \sigma_*)$ is such that $\pi(\sigma_*^{ab}) - \pi(\sigma_*) > 0$, so, this pair gives $a$ an excess score that is strictly positive. We conclude that $a$ has strictly higher Borda score than $b$, hence $b$ is not selected by Borda count. 

\emph{Copeland.} Let $c\in A\setminus \{a,b\}$. In a pairwise election between $b$ and $c$, the total weight of rankings that place $b$ over $c$ is
\small
\begin{align*}
\sum_{\sigma \in \mathcal{S}_A:\, b \succ_\sigma c} \pi(\sigma) = \sum_{\sigma \in \mathcal{S}_A:\, (b \succ_\sigma c) \land (a \succ_\sigma c)} \pi(\sigma) + \sum_{\sigma \in \mathcal{S}_A:\, (b \succ_\sigma c) \land (c \succ_\sigma a)} \pi(\sigma).
\end{align*}
\normalsize
For the rankings in the second summation (on the right-hand side), we have $b \succ_\sigma a$ by transitivity. Hence, $\pi(\sigma) \leq \pi(\sigma^{ab})$ for such rankings. Therefore,
\small 
\begin{align*}
\sum_{\sigma \in \mathcal{S}_A:\, b \succ_\sigma c} \pi(\sigma) 
&\leq \sum_{\sigma \in \mathcal{S}_A:\, (b \succ_\sigma c) \land (a \succ_\sigma c)} \pi(\sigma) + \sum_{\sigma \in \mathcal{S}_A:\, (b \succ_\sigma c) \land (c \succ_\sigma a)} \pi(\sigma^{ab})\\
& = \sum_{\sigma \in \mathcal{S}_A:\, (b \succ_\sigma c) \land (a \succ_\sigma c)} \pi(\sigma) + \sum_{\sigma' \in \mathcal{S}_A:\, (a \succ_{\sigma'} c) \land (c \succ_{\sigma'} b)} \pi(\sigma')\\
& = \sum_{\sigma \in \mathcal{S}_A:\, a \succ_\sigma c} \pi(\sigma).
\end{align*}
\normalsize
In summary, we have $$\sum_{\sigma \in \mathcal{S}_A:\, b \succ_\sigma c} \pi(\sigma) \leq \sum_{\sigma \in \mathcal{S}_A:\, a \succ_\sigma c} \pi(\sigma).$$ Hence, if $b$ beats $c$ in a pairwise competition, then so does $a$. 
%And if $b$ ties with $c$, then either $a$ ties with $c$, or $a$ beats $c$. 
Therefore, the Copeland score of $a$ (due to all alternatives other than $a$ and $b$) is at least as high as that of $b$. Further, in a pairwise competition between $a$ and $b$, the weight of rankings that position $a$ above $b$ is $\sum_{\sigma \in \mathcal{S}_A:\, a \succ_\sigma b} \pi(\sigma)$ and the weight of those that prefer $b$ over $a$ is $\sum_{\sigma \in \mathcal{S}_A:\, b \succ_\sigma a} \pi(\sigma)$. But, because $\pi(\sigma) \geq \pi(\sigma^{ab})$ for any $\sigma$ with $a \succ_\sigma b$, and $\pi(\sigma_*^{ab}) > \pi(\sigma_*)$, $a$ beats $b$. Therefore, $a$ has a strictly higher Copeland score than $b$, and $b$ is not selected by Copeland. 

\medskip
\noindent\textbf{Case 2:} $b \SwD_\pi a$. In this case, $a \SwD_\pi b$ and $b \SwD_\pi a$. This means that for all $\sigma \in \mathcal{S}_A$, we have $\pi(\sigma) = \pi(\sigma^{ab})$. In other words, $\tau(\pi) = \pi$, where $\tau$ is the permutation that swaps $a$ and $b$. Both Borda count and Copeland are neutral SCCs. So, we have $\tau(f(\pi)) = f(\tau(\pi))$, which is in turn equal to $f(\pi)$. Hence, $a$ is selected if and only if $b$ is selected. 

We conclude that both conditions of Definition~\ref{def:strong-swd} are satisfied by Borda count and Copeland.
\end{proof}

\begin{lemma} \label{lem:swap-across-subsets}
Let $\Pi$ be a consistent permutation process that is SwD-compatible. Then, for any finite subset of alternatives $A\subseteq \X$, $\left(\SwD_{\Pi(A)}\right) =  \left(\SwD_\Pi|_A\right)$. 
\end{lemma}

In words, as long as $\Pi$ is consistent and SwD-compatible, marginalizing out some alternatives from a profile does not remove or add any swap-dominance relations.

\begin{proof}[Proof of Lemma~\ref{lem:swap-across-subsets}]
We first show that for any $B \subseteq A \subseteq \X$, $\left(\SwD_{\Pi(A)}|_B\right) = \left(\SwD_{\Pi(B)}\right)$.

Let $a, b \in B$ such that $a \SwD_{\Pi(A)} b$. Now, let $\sigma \in \mathcal{S}_B$ be an arbitrary ranking such that $a \succ_\sigma b$. Also, let $\pi_B$ denote $\Pi(B)$ and $\pi_A$ denote $\Pi(A)$. Then, since $\Pi$ is consistent,
\begin{align*}
\pi_B(\sigma) = \sum_{\sigma_2 \in \mathcal{S}_A:\, \sigma_2|_B = \sigma} \pi_A(\sigma_2).
\end{align*}
Now, for $\sigma_2 \in \mathcal{S}_A$ such that $\sigma_2|_B = \sigma$, we have $a \succ_{\sigma_2} b$ and therefore $\pi_A(\sigma_2) \geq \pi_A(\sigma_2^{ab})$ (because $a \SwD_{\Pi(A)} b$). It follows that
\begin{align*}
\pi_B(\sigma) &= \sum_{\sigma_2 \in \mathcal{S}_A: \sigma_2|_B = \sigma} \pi_A(\sigma_2)
\geq \sum_{\sigma_2 \in \mathcal{S}_A: \sigma_2|_B = \sigma} \pi_A(\sigma_2^{ab})\\
&= \sum_{\sigma_2' \in \mathcal{S}_A: \sigma_2'|_B = \sigma^{ab}} \pi_A(\sigma_2')\\
&= \pi_B(\sigma^{ab}).
\end{align*}
Therefore, $a \SwD_{\Pi(B)} b$, that is,  $\left(\SwD_{\Pi(A)}|_B\right) \subseteq \left(\SwD_{\Pi(B)}\right)$.

Next we show that $\left(\SwD_{\Pi(B)}\right) \subseteq \left(\SwD_{\Pi(A)}|_B\right)$. Let $a, b \in B$ such that $a \SwD_{\Pi(B)} b$. Suppose for the sake of contradiction that $a \notSwD_{\Pi(A)} b$. This implies that $a \notSwD_{\Pi} b$. However, $\SwD_\Pi$ is a total preorder because $\Pi$ is SwD-compatible (by definition). It follows that $b \SwD_{\Pi} a$, and, in particular, $b \SwD_{\Pi(A)} a$ and $b \SwD_{\Pi(B)} a$.

As before, let $\pi_A$ denote $\Pi(A)$ and $\pi_B$ denote $\Pi(B)$. Because $a \notSwD_{\Pi(A)} b$, there exists $\sigma_* \in \mathcal{S}_A$ with $a \succ_{\sigma_*} b$ such that $\pi_A(\sigma_*) < \pi_A(\sigma^{ab}_*)$. Moreover, because $a \SwD_{\Pi(B)} b$ and $b \SwD_{\Pi(B)} a$, it holds that $\pi_B(\sigma_*|_B) = \pi_B\left((\sigma_*|_B)^{ab}\right)$. The consistency of $\Pi$ then implies that
\begin{align} \label{eqn:two-sums}
\sum_{\sigma_1 \in \mathcal{S}_A:\, \sigma_1|_B = \sigma_*|_B} \pi_A(\sigma_1) = \sum_{\sigma_2 \in \mathcal{S}_A:\, \sigma_2|_B = (\sigma_*|_B)^{ab}} \pi_A(\sigma_2).
\end{align}
Note $\sigma_1 = \sigma_*$ is a ranking that appears on the left-hand side of Equation~\eqref{eqn:two-sums}, and $\sigma_2 = \sigma_*^{ab}$ is a ranking that appears on the right-hand side. Furthermore, we know that $\pi_A(\sigma_*) < \pi_A(\sigma^{ab}_*)$. It follows that there exists $\sigma' \in \mathcal{S}_A$ with $\sigma'|_B = \sigma_*|_B$ such that $\pi_A(\sigma') > \pi_A\left((\sigma')^{ab}\right)$. Also, since $\sigma'|_B = \sigma_*|_B$, it holds that $a \succ_{\sigma'} b$. We conclude that it cannot be the case that $b \SwD_{\Pi(A)} a$, leading to a contradiction. Therefore, if $a \SwD_{\Pi(B)} b$, then $a \SwD_{\Pi(A)} b$, i.e., $\left(\SwD_{\Pi(B)}\right) \subseteq \left(\SwD_{\Pi(A)}|_B\right)$.

We next prove the lemma itself, i.e., that $\left(\SwD_{\Pi(A)}\right) =  \left(\SwD_\Pi|_A\right)$ . Firstly, for $a,b \in A$, if $a \SwD_\Pi b$, then $a \SwD_{\Pi(A)} b$ by definition. So, we easily get $\left(\SwD_\Pi|_A\right) \subseteq \left(\SwD_{\Pi(A)}\right)$.

In the other direction, let $a,b \in A$ such that $a \SwD_{\Pi(A)} b$. Let $C$ be an arbitrary set of alternatives containing $a$ and $b$. From what we have shown above, we have $\left(\SwD_{\Pi(A)}|_{\{a,b\}}\right) = \left(\SwD_{\Pi(\{a,b\})}\right)$. Also, $\left(\SwD_{\Pi(C)}|_{\{a,b\}}\right) = \left(\SwD_{\Pi(\{a,b\})}\right)$. This gives us $\left(\SwD_{\Pi(A)}|_{\{a,b\}}\right) = \left(\SwD_{\Pi(C)}|_{\{a,b\}}\right)$. Hence, $a \SwD_{\Pi(C)} b$, and this is true for every such subset $C$. We conclude that $a \SwD_\Pi b$, that is, $\left(\SwD_{\Pi(A)}\right) \subseteq \left(\SwD_\Pi|_A\right)$.
\end{proof}

\begin{lemma} \label{lem:strong-apply}
Let $f$ be a strongly SwD-efficient anonymous SCC, and let $\Pi$ be a consistent permutation process that is SwD-compatible. Then for any finite subset of alternatives $A$, $f(\Pi(A)) = \{a \in A: a \SwD_\Pi b \text{ \ for all } b \in A\}$.
\end{lemma}

\begin{proof}
Let $A$ be an arbitrary finite subset of alternatives. Since strong SwD-efficiency implies SwD-efficiency, Theorem~\ref{thm:apply-rule-proc} gives us
$$f(\Pi(A)) \supseteq \{a \in A: a \SwD_\Pi b \text{ \ for all } b \in A\}.$$

In the other direction, let $a \in f(\Pi(A))$. Suppose for the sake of contradiction that there exists $b \in A$ such that $a \notSwD_\Pi b$. Since $\SwD_\Pi$ is a total preorder, it follows that $b \SwD_\Pi a$. By Lemma~\ref{lem:swap-across-subsets}, it holds that $\left(\SwD_{\Pi(A)}\right) =  \left(\SwD_\Pi|_A\right)$, and therefore $a \notSwD_{\Pi(A)} b$ and $b \SwD_{\Pi(A)} a$. But, since $f$ is strongly SwD-efficient, it follows that $a \notin f(\Pi(A))$, which contradicts our assumption. Hence,
$$f(\Pi(A)) \subseteq \{a \in A: a \SwD_\Pi b \text{ \ for all } b \in A\},$$
and we have the desired result.
\end{proof}

\begin{theorem}	\label{thm:total-swap-self-cons}
Let $\Pi$ be a consistent permutation process that is SwD-compatible, and let $f$ be a strongly SwD-efficient anonymous SCC. Then the pair $(\Pi, f)$ is stable.
\end{theorem}

\begin{proof}
Consider an arbitrary subset of alternatives $A$, and let $B\subseteq A$. By Lemma~\ref{lem:strong-apply}, $f(\Pi(A)) = \{a \in A: a \SwD_\Pi b \text{ \ for all } b \in A\}$, and similarly for $B$. Suppose $f(\Pi(A)) \cap B \neq \phi$, and let $a \in f(\Pi(A)) \cap B$, i.e. $a \in f(\Pi(A))$ and $a \in B$. This means that $a \SwD_\Pi b$ for all $b \in A$, and, therefore $a \SwD_\Pi b$ for all $b \in B$. We conclude that $a \in f(\Pi(B))$, and hence $f(\Pi(A)) \cap B \subseteq f(\Pi(B))$.

In the other direction, let $a \in f(\Pi(B))$. This means that $a \SwD_\Pi b$ for all $b \in B$. Suppose for the sake of contradiction that $a \notin f(\Pi(A))$. This means that there exists $c \in A$ such that $a \notSwD_\Pi c$. We assumed $f(\Pi(A)) \cap B \neq \phi$, so let $d\in f(\Pi(A)) \cap B$. Then, $d \SwD_\Pi c$. In summary, we have $d \SwD_\Pi c$ and $a \notSwD_\Pi c$, which together imply that $a \notSwD_\Pi d$ (otherwise, it would violate transitivity). But $d \in B$, leading to $a \notin f(\Pi(B))$, which contradicts the assumption. Therefore, indeed $a \in f(\Pi(A))$, and it holds that $f(\Pi(B)) \subseteq f(\Pi(A)) \cap B$, as long as $f(\Pi(A)) \cap B \neq \phi$.
\end{proof}

We are now ready to prove Theorem~\ref{thm:stability}.

\begin{proof}[Proof of Theorem~\ref{thm:stability}]
From Lemma~\ref{lem:comm-strong-SwD}, Borda count and Copeland are strongly SwD-efficient. Lemmas~\ref{lem:dom-implies-SwD} and \ref{lem:TM-PL-dom} imply that when $\Pi$ is the TM or PL process, $\SwD_\Pi$ is a total preorder. In particular, $a \SwD_\Pi b$ if $\mu_a \geq \mu_b$. Hence, $\Pi$ is SwD-compatible. Therefore, by Theorem~\ref{thm:total-swap-self-cons}, the pair $(\Pi, f)$ is stable.
\end{proof} 

\section{Instantiation of Our Approach}
\label{sec:instantiation}

In this section, we instantiate our approach for ethical decision making, as outlined in Section~\ref{sec:intro}. In order to present a concrete algorithm, we consider a specific permutation process, namely the TM process with a linear parameterization of the utility process parameters as a function of the alternative features.

Let the set of alternatives be given by $\X \subseteq \mathbb{R}^d$, i.e. each alternative is represented by a vector of $d$ features. Furthermore, let $N$ denote the total number of voters. Assume for now that the data-collection step (Step I) is complete, i.e., we have some pairwise comparisons for each voter; we will revisit this step in Section~\ref{sec:evaluation}.

\medskip
\noindent\textbf{Step II: Learning.} For each voter, we learn a TM process using his pairwise comparisons to represent his preferences. We assume that the mode utility of an alternative $x$ depends linearly on its features, i.e., $\mu_x = \myvec{\beta}^\intercal x$. Note that we do not need an intercept term, since we care only about the relative ordering of utilities. Also note that the parameter $\myvec{\beta} \in \mathbb{R}^d$ completely describes the TM process, and hence the parameters $\myvec{\beta}_1, \myvec{\beta}_2, \cdots \myvec{\beta}_N$ completely describe the models of all voters.

Next we provide a computationally efficient method for learning the parameter $\myvec{\beta}$ for a particular voter. Let $(X_1, Z_1), (X_2, Z_2), \cdots, (X_n, Z_n)$ denote the pairwise comparison data of the voter. Specifically, the ordered pair $(X_j, Z_j)$ denotes the $j^{th}$ pair of alternatives compared by the voter, and the fact that the voter chose $X_j$ over $Z_j$. We use maximum likelihood estimation to estimate $\myvec{\beta}$. The log-likelihood function is
\begin{align*}
\mathcal{L}(\myvec{\beta})& = \log\left[\prod_{j=1}^n P(X_j \succ Z_j; \myvec{\beta})\right]\\
&= \sum_{j=1}^n \log P(U_{X_j} > U_{Z_j}; \myvec{\beta})\\
& = \sum_{j=1}^n \log \Phi\left(\myvec{\beta}^\intercal(X_j - Z_j)\right),
\end{align*}
where $\Phi$ is the cumulative distribution function of the standard normal distribution, and the last transition holds because $U_x \sim \mathcal{N}(\myvec{\beta}^\intercal x, \frac{1}{2})$. Note that the standard normal CDF $\Phi$ is a log-concave function. This makes the log-likelihood concave in $\myvec{\beta}$, hence we can maximize it efficiently.

\medskip
\noindent\textbf{Step III: Summarization.} After completing Step II, we have $N$ TM processes represented by the parameters $\myvec{\beta}_1, \myvec{\beta}_2, \cdots \myvec{\beta}_N$. In Step III, we bundle these individual models into a single permutation process $\hat{\Pi}$, which, in the current instantiation, is also a TM process with parameter $\hat{\myvec{\beta}}$ (see Section~\ref{sec:disc} for a discussion of this point). We perform this step because we must be able to make decisions \emph{fast}, in Step IV. For example, in the autonomous vehicle domain, the AI would only have a split second to make a decision in case of emergency; aggregating information from millions of voters \emph{in real time} will not do. By contrast, Step III is performed offline, and provides the basis for fast aggregation. 

Let $\Pi^{\myvec{\beta}}$ denote the TM process with parameter ${\myvec{\beta}}$. Given a finite subset of alternatives $A\subseteq \X$, the anonymous preference profile generated by the model of voter $i$ is given by $\Pi^{{\myvec{\beta}}_i}(A)$. Ideally, we would like the summary model to be such that the profile generated by it, $\hat{\Pi}(A)$, is as close as possible to $\Pi^*(A) = \frac{1}{N} \sum_{i=1}^N \Pi^{{\myvec{\beta}}_i}(A)$, the mean profile obtained by giving equal importance to each voter. However, there does not appear to be a straightforward method to compute the ``best'' $\hat{\myvec{\beta}}$, since the profiles generated by the TM processes do not have an explicit form. Hence, we use utilities as a proxy for the quality of $\hat{\myvec{\beta}}$. Specifically, we find $\hat{\myvec{\beta}}$ such that the summary model induces utilities that are as close as possible to the mean of the utilities induced by the per-voter models, i.e., we want $U^{\hat{\myvec{\beta}}}_x$ to be as close as possible (in terms of KL divergence) to $\frac{1}{N} \sum_{i=1}^N U^{\myvec{\beta_i}}_x$ for each $x \in \X$, where $U^{\myvec{\beta}}_x$ denotes the utility of $x$ under TM process with parameter $\myvec{\beta}$. This is achieved by taking $\hat{\myvec{\beta}} = \frac{1}{N}\sum_{i=1}^N \myvec{\beta}_i$, as shown by the following proposition.

\begin{proposition}
\label{prop:KL}
The vector $\myvec{\beta} = \frac{1}{N}\sum_{i=1}^N \myvec{\beta}_i$ minimizes $KL\left(\frac{1}{N} \sum_{i=1}^N U^{\myvec{\beta}_i}_x \big\| U^{\myvec{\beta}}_x\right)$ for any $x \in \X$.
\end{proposition}

\begin{proof}
Let $\bar{\myvec{\beta}} = \frac{1}{N} \sum_{i=1}^N \myvec{\beta}_i$. We know that $U_x^{\myvec{\beta}}$ denotes the utility of $x$ under the TM process with parameter $\myvec{\beta}$. So, $U^{\myvec{\beta}}_x \sim \mathcal{N}(\myvec{\beta}^\intercal x, \frac{1}{2})$. Let its density be given by $q_{x, \myvec{\beta}}(\cdot)$. Also, $U_x^{\myvec{\beta}_i} \sim \mathcal{N}(\myvec{\beta}_i^\intercal x, \frac{1}{2})$. Hence, $$\frac{1}{N} \sum_{i=1}^N U^{\myvec{\beta}_i}_x \sim \mathcal{N}(\bar{\myvec{\beta}}^\intercal x, \frac{1}{2N}).$$ Let its density function be denoted by $p_x(\cdot)$. Then
$$KL(p_x \| q_{x,\myvec{\beta}}) = \int p_x(t) \log p_x(t) dt - \int p_x(t) \log q_{x, \myvec{\beta}}(t) dt.$$
Since the first term does not depend on $\myvec{\beta}$, let us examine the second term:
\begin{align*}
- \int p_x(t) \log q_{x,\myvec{\beta}}(t) dt
&= - \int p_x(t) \log\left( \frac{1}{\sqrt{\pi}} \exp\left(-(t - \myvec{\beta}^\intercal x)^2\right) \right) dt\\
&= - \int p_x(t) \left[-\frac{1}{2}\log(\pi) - (t-\myvec{\beta}^\intercal x)^2 \right] dt\\
&= \frac{1}{2}\log(\pi) \left(\int p_x(t) dt\right) + \int p_x(t) \left(t^2 + (\myvec{\beta}^\intercal x)^2 -2 t \myvec{\beta}^\intercal x \right) dt\\
&= \frac{1}{2}\log(\pi) + \left( \int t^2 p_x(t) dt + (\myvec{\beta}^\intercal x)^2 \int p_x(t) dt -2 \myvec{\beta}^\intercal x \int t p_x(t) dt \right)\\
&= \frac{1}{2}\log(\pi) + \left( \left(\frac{1}{2N} + (\bar{\myvec{\beta}}^\intercal x)^2\right) + (\myvec{\beta}^\intercal x)^2 -2 \myvec{\beta}^\intercal x (\bar{\myvec{\beta}}^\intercal x) \right)\\
&= \frac{1}{2}\log(\pi) + \frac{1}{2N} + \left(\bar{\myvec{\beta}}^\intercal x - \myvec{\beta}^\intercal x\right)^2.
\end{align*}
This term is minimized at $\myvec{\beta} = \bar{\myvec{\beta}}$ for any $x$, and therefore $KL(\frac{1}{N} \sum_{i=1}^N U^{\myvec{\beta}_i}_x \big\| U^{\myvec{\beta}}_x)$ is minimized at that value as well.
\end{proof}

\medskip
\noindent\textbf{Step IV: Aggregation.} As a result of Step III, we have exactly one (summary) TM process $\hat{\Pi}$ (with parameter $\hat{\myvec{\beta}} = \bar{\myvec{\beta}}$) to work with at runtime. Given a finite set of alternatives $A=\{x_1, x_2, \cdots, x_m\}$, we must aggregate the preferences represented by the anonymous preference profile $\hat{\Pi}(A)$. This is where the machinery of Section~\ref{sec:aggregation} comes in: We simply need to select an alternative that has maximum mode utility among $\hat{\myvec{\beta}}^\intercal x_1, \hat{\myvec{\beta}}^\intercal x_2, \cdots, \hat{\myvec{\beta}}^\intercal x_m$. Such an alternative would be selected by any anonymous SCC that is monotonic and neutral, when applied to $\hat{\Pi}(A)$, as shown by Theorem~\ref{thm:main}. Moreover, this aggregation method is equivalent to applying the Borda count or Copeland SCCs. Hence, we also have the desired stability property, as shown by Theorem~\ref{thm:stability}.

\section{Implementation and Evaluation}
\label{sec:evaluation}

In this section, we implement the algorithm presented in Section~\ref{sec:instantiation}, and empirically evaluate it. We start with an implementation on synthetic data, which allows us to effectively validate both Steps~II and III of our approach. We then describe the Moral Machine dataset mentioned in Section~\ref{sec:intro}, present the implementation of our algorithm on this dataset, and evaluate the resultant system for ethical decision making in the autonomous vehicle domain (focusing on Step~III).

\subsection{Synthetic Data}
\label{subsec:synth}

\medskip
\noindent\textbf{Setup.} We represent the preferences of each voter using a TM process. Let $\myvec{\beta}_i$ denote the true parameter corresponding to the model of voter $i$. We sample $\myvec{\beta}_i$ from $\mathcal{N}(\myvec{m}, I_d)$ (independently for each voter $i$), where each mean $m_j$ is sampled independently from the uniform distribution $\mathcal{U}(-1,1)$, and the number of features is $d=10$. 

In each instance (defined by a subset of alternatives $A$ with $|A| = 5$), the desired winner is given by the application of Borda count to the mean of the profiles of the voters. In more detail, we compute the anonymous preference profile of each voter $\Pi^{\myvec{\beta}_i}(A)$, and then take a mean across all the voters to obtain the desired profile $\frac{1}{N} \sum_{i=1}^N \Pi^{\myvec{\beta}_i}(A)$. We then apply Borda count to this profile to obtain the winner. Note that, since we are dealing with TM processes, we cannot explicitly construct $\Pi^{\myvec{\beta}_i}(A)$; we therefore estimate it by sampling rankings according to the TM process of voter $i$. 

% In each instance, we check whether the foregoing ``ground-truth'' outcome coincides with the winner selected by using Step II or Step III. We then compute the accuracy as the fraction of instances in which the two outcomes match. %\rncomment{I see that you removed the fact that we use $100$ instances to compute this accuracy. Is it not required?}
%
\medskip
\noindent\textbf{Evaluation of Step II (Learning).} In practice, the algorithm does not have access to the true parameter $\myvec{\beta}_i$ of voter $i$, but only to pairwise comparisons, from which we learn the parameters. Thus we compare the computation of the winner (following the approach described above) using the true parameters, and using the learned parameters as in Step II. We report the accuracy as the fraction of instances, out of $100$ test instances, in which the two outcomes match. 

To generate each pairwise comparison of voter $i$, for each of $N=20$ voters, we first sample two alternatives $x_1$ and $x_2$ independently from $\mathcal{N}(\myvec{0}, I_d)$. Then, we sample their utilities $U_{x_1}$ and $U_{x_2}$ from $\mathcal{N}({\myvec{\beta}}_i^\intercal x_1, \frac{1}{2})$ and $\mathcal{N}({\myvec{\beta}}_i^\intercal x_2, \frac{1}{2})$, respectively. Of course, the voter prefers the alternative with higher sampled utility. Once we have the comparisons, we learn the parameter $\myvec{\beta}_i$ by computing the MLE (as explained in Step II of Section~\ref{sec:instantiation}). 
% Then, given a instance, we compute the winner by using the same approach as described above (except that we now use the estimated parameters rather than the true ones). %This helps us empirically evaluate Step II in isolation.
%
In our results, we vary the number of pairwise comparisons per voter and compute the accuracy to obtain the learning curve shown in Figure~\ref{fig:first}. Each datapoint in the graph is averaged over $50$ runs. Observe that the accuracy quickly increases as the number of pairwise comparisons increases, and with just $30$ pairwise comparisons we achieve an accuracy of $84.3\%$. With $100$ pairwise comparisons, the accuracy is $92.4\%$.

\medskip
\noindent\textbf{Evaluation of Step III (Summarization).} %Even if we had access to the true parameters $\myvec{\beta}_i$, we would not be able to compute the winner  compute the ground truth winner (at run-time) because it'll take too long to construct the mean profile $\frac{1}{N} \sum_{i=1}^N \Pi^{\myvec{\beta}_i}(A)$. As mentioned in Step III (of Section~\ref{subsec:instantiation}), we circumvent this issue by summarizing the models. 
To evaluate Step III, we assume that we have access to the true parameters $\myvec{\beta}_i$, and wish to determine the accuracy loss incurred in the summarization step, where we summarize the individual TM models into a single TM model. As described in Section~\ref{sec:instantiation}, we compute $\bar{\myvec{\beta}} = \frac{1}{N} \sum_{i=1}^N \myvec{\beta}_i$, and, given a subset $A$ (which again has cardinality 5), we aggregate using Step IV, since we now have just one TM process. For each instance, we contrast our computed winner with the desired winner as computed previously.
We vary the number of voters and compute the accuracy to obtain Figure~\ref{fig:second}. The accuracies are averaged over $50$ runs. Observe that the accuracy increases to $93.9\%$ as the number of voters increases. In practice we expect to have access to thousands, even millions, of votes (see Section~\ref{subsec:moral}). We conclude that, surprisingly, the expected loss in accuracy due to summarization is quite small.

\medskip
\noindent\textbf{Robustness.} Our results are robust to the choice of parameters, as we demonstrate in Appendix~\ref{app:experiments}.

\begin{figure}[t]
\centering
\begin{minipage}{0.5\textwidth}
  \centering
  \includegraphics[width=.9\linewidth]{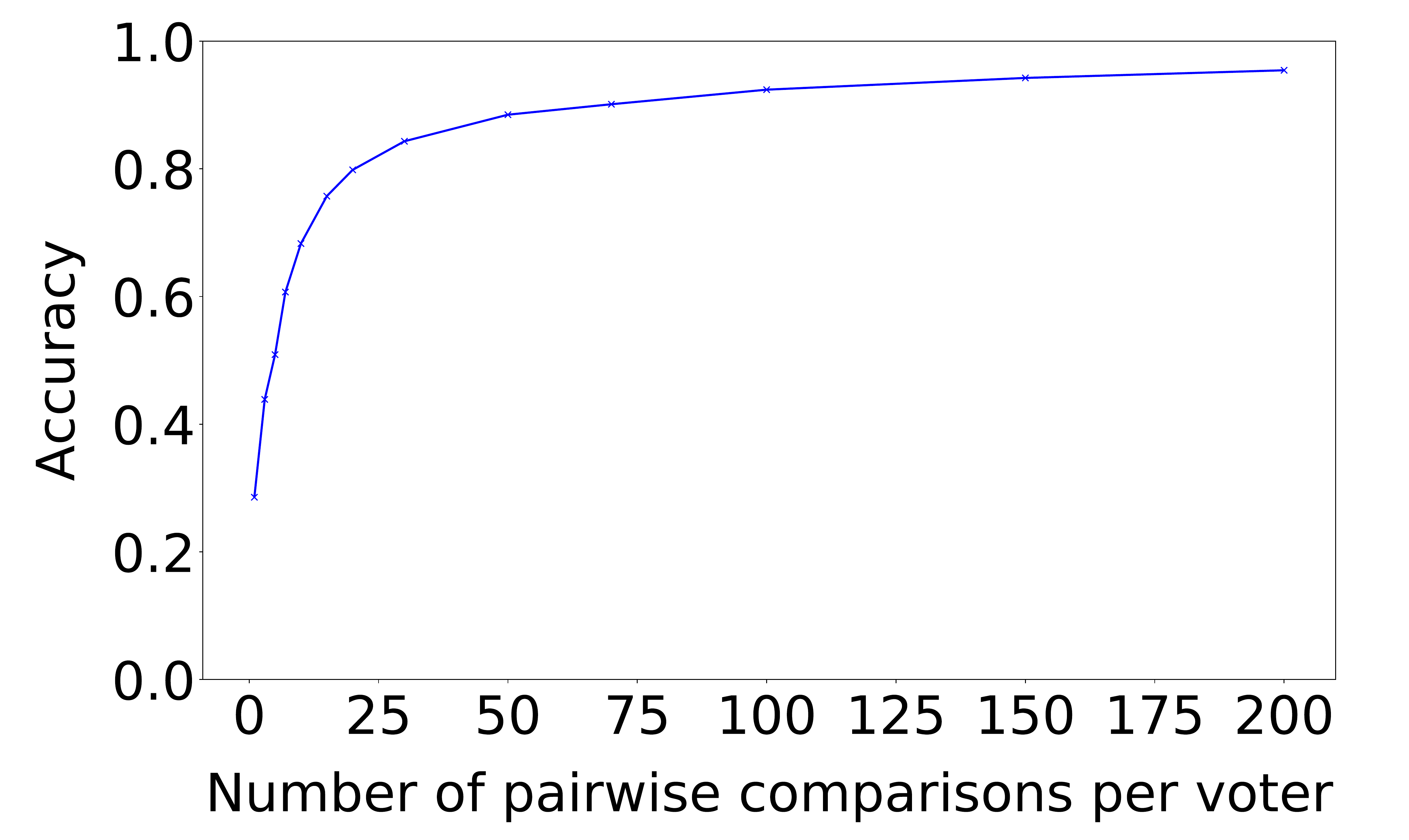}
  \captionof{figure}{Accuracy of Step II (synthetic data) }
  \label{fig:first}
\end{minipage}%
\begin{minipage}{0.5\textwidth}
  \centering
  \includegraphics[width=.9\linewidth]{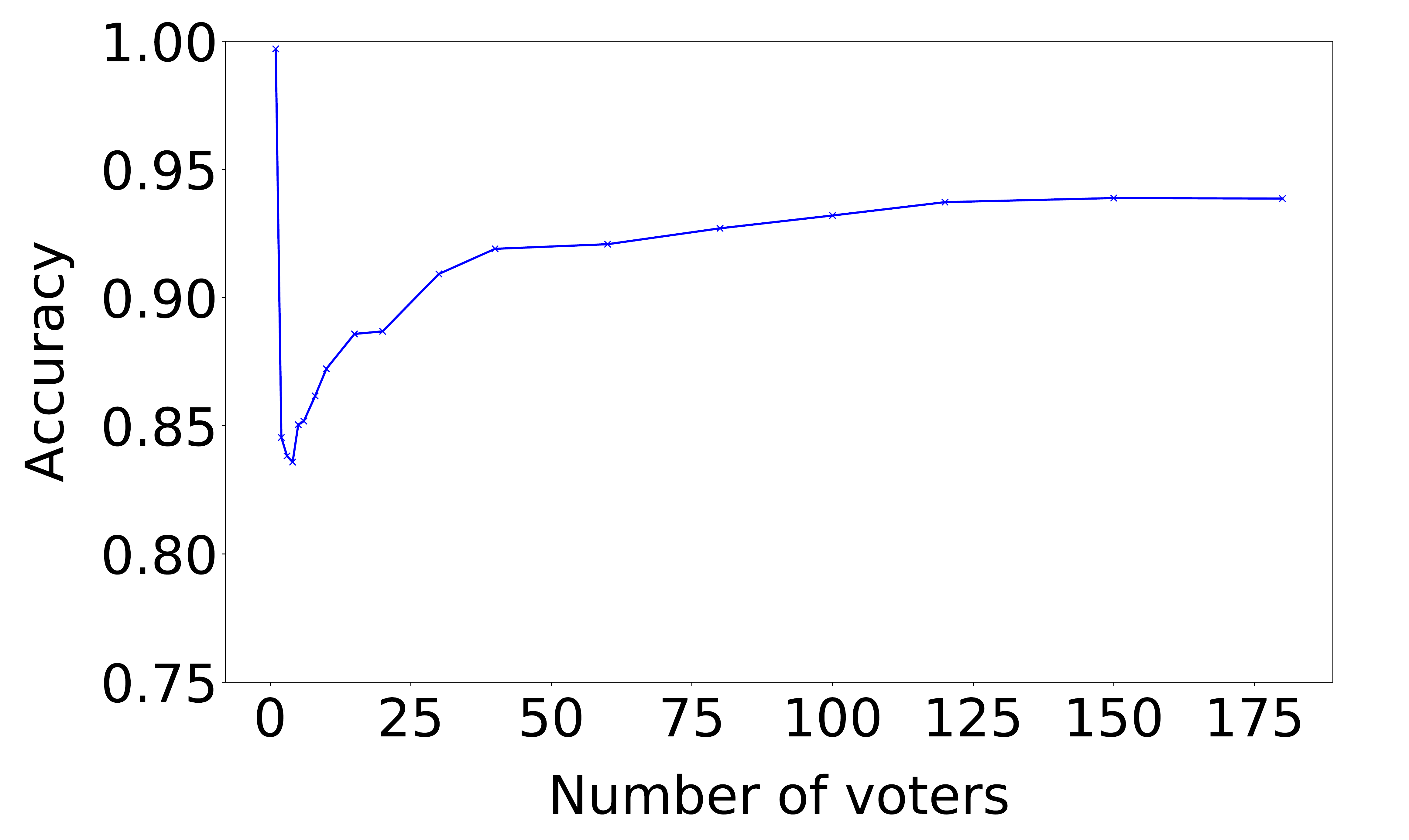}
  \captionof{figure}{Accuracy of Step III (synthetic data)}
  \label{fig:second}
\end{minipage}
\end{figure}

\subsection{Moral Machine Data}
\label{subsec:moral}

Moral Machine is a platform for gathering data on human perception of the moral acceptability of decisions made by autonomous vehicles faced with choosing which humans to harm and which to save. The main interface of Moral Machine is the Judge mode. This interface generates sessions of random moral dilemmas. In each session, a user is faced with 13 instances. Each instance features an autonomous vehicle with a brake failure, facing a moral dilemma with two possible alternatives, that is, each instance is a pairwise comparison. Each of the two alternatives corresponds to sacrificing the lives of one group of characters to spare those of another group of characters. Figure \ref{fig:MM} shows an example of such an instance. Respondents choose the outcome that they prefer the autonomous vehicle to make. 

%At the end of the 13 instances, the user is shown a summary of their decisions, and how they compare to other users. Additionally, the user can optionally fill a demographic survey. The user might then decide to do another session (users can do as many sessions as they want).

\begin{figure}[t]
\centering
\includegraphics[width=0.5\textwidth]{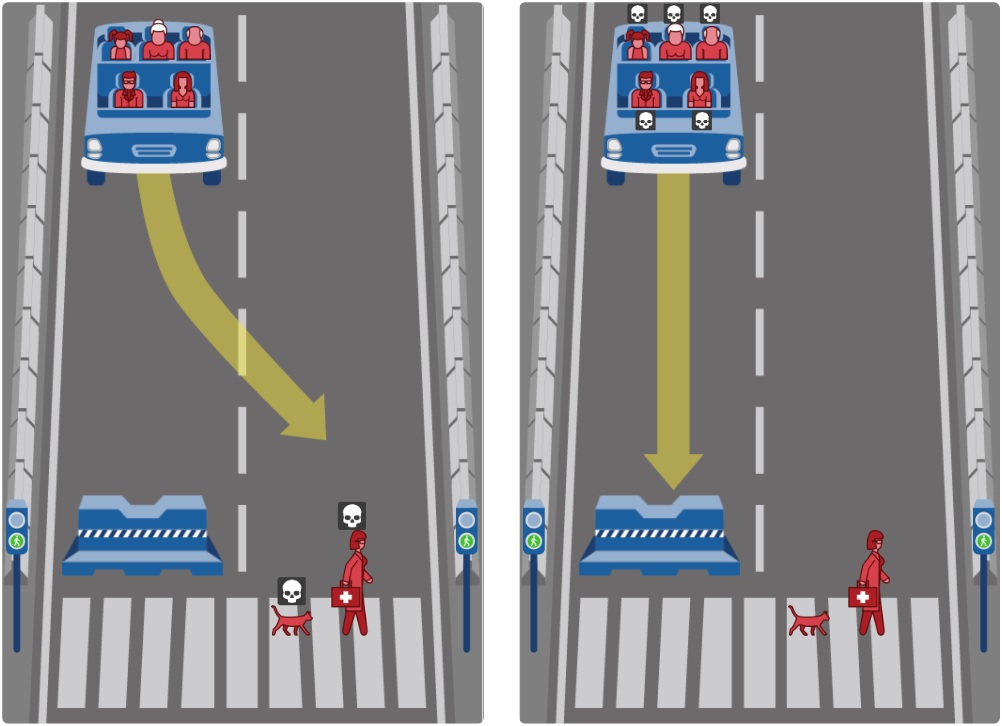}
\caption{\emph{Moral Machine}\,---\,Judge interface. This particular choice is between a group of pedestrians that includes a female doctor and a cat crossing on a green light, and a group of passengers including a woman, a male executive, an elderly man, an elderly woman, and a girl.}
\label{fig:MM}
\end{figure} 

Each alternative is characterized by 22 features: relation to the autonomous vehicle (passengers or pedestrians), legality (no legality, explicitly legal crossing, or explicitly illegal crossing), and counts of 20 character types, including ones like man, woman, pregnant woman, male athlete, female doctor, dog, etc. When sampling from the 20 characters, some instances are generated to have an easy-to-interpret tradeoff with respect to some dimension, such as gender (males on one side vs. females on the other), age (elderly vs. young), fitness (large vs. fit), etc., while other instances have groups consisting of completely randomized characters being sacrificed in either alternative. Alternatives with all possible combinations of these features are considered, except for the legality feature in cases when passengers are sacrificed. In addition, each alternative has a derived feature, ``number of characters,'' which is simply the sum of counts of the $20$ character types (making $d=23$).

As mentioned in Section~\ref{sec:intro}, the Moral Machine dataset consists of preference data from 1,303,778 voters, amounting to a total of 18,254,285 pairwise comparisons. We used this dataset to learn the $\myvec{\beta}$ parameters of all 1.3 million voters (Step II, as given in Section~\ref{sec:instantiation}). Next, we took the mean of all of these $\myvec{\beta}$ vectors to obtain $\hat{\myvec{\beta}}$ (Step III). This gave us an implemented system, which can be used to make real-time choices between any finite subset of alternatives. 

Importantly, the methodology we used, in Section~\ref{subsec:synth}, to evaluate Step II on the synthetic data cannot be applied to the Moral Machine data, because we do not know which alternative would be selected by aggregating the preferences of the actual 1.3 million voters over a subset of alternatives. However, we can apply a methodology similar to that of Section~\ref{subsec:synth} in order to evaluate Step III. 
%Now, since Step III (summarizing by taking mean of the $\myvec{\beta}$s) is an approximate method, we demonstrate that such summarization works even for the Moral Machine data. In particular, we evaluate Step III as in Section~\ref{subsec:synth}, but now using $\myvec{\beta}$s learnt from the Moral Machine data. We also show robustness as the number of alternatives in each instance is increased. 
Specifically, as in Section~\ref{subsec:synth}, we wish to compare the winner obtained using the summarized model, with the winner obtained by applying Borda count to the mean of the anonymous preference profiles of the voters. 

An obstacle is that now we have a total of $1.3$ million voters, and hence it would take an extremely long time to calculate the anonymous preference profile of each voter and take their mean (this was the motivation for having Step III in the first place). So, instead, we estimate the mean profile by sampling rankings, i.e., we sample a voter $i$ uniformly at random, and then sample a ranking from the TM process of voter $i$; such a sampled ranking is an i.i.d.~sample from the mean anonymous profile. Then, we apply Borda count as before to obtain the desired winner (note that this approach is still too expensive to use in real time). The winner according to the summarized model is computed exactly as before, and is just as efficient even with $1.3$ million voters. 

Using this methodology, we computed accuracy on $3000$ test instances, i.e., the fraction of instances in which the two winners match. Figure~\ref{fig:StepIII_MMdata} shows the results as the number of alternatives per instance is increased from $2$ to $10$. Observe that the accuracy is as high as $98.2\%$ at $2$ alternatives per instance, and gracefully degrades to $95.1\%$ at $10$.

\begin{figure}[t]
  \centering
  \includegraphics[width=0.5\textwidth]{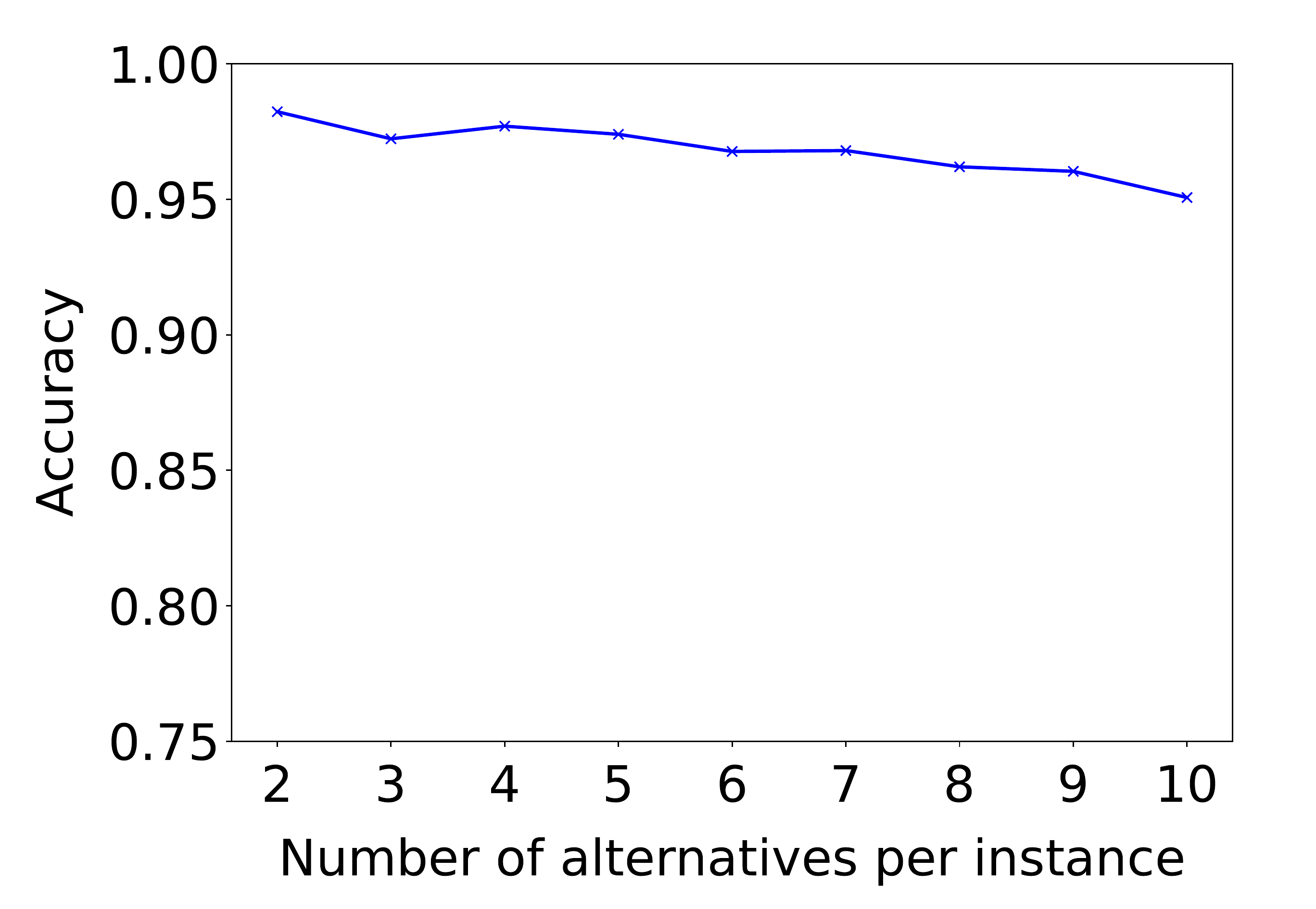}
  \caption{Accuracy of Step III (Moral Machine data)}
  \label{fig:StepIII_MMdata}
\end{figure}

\section{Discussion}
\label{sec:disc}

The design of intelligent machines that can make ethical decisions is, arguably, one of the hardest challenges in AI. We do believe that our approach takes a significant step towards addressing this challenge. In particular, the implementation of our algorithm on the Moral Machine dataset has yielded a system which, arguably, can make \emph{credible} decisions on ethical dilemmas in the autonomous vehicle domain (when all other options have failed). But this paper is clearly not the end-all solution.

\subsection{Limitations}

While the work we presented has some significant limitations, we view at least some of them as shortcomings of the current (proof-of-concept) implementation, rather than being inherent to the approach itself, as we explain below. 

First, Moral Machine users may be poorly informed about the dilemmas at hand, or may not spend enough time thinking through the options, potentially leading\,---\,in some cases\,---\,to inconsistent answers and poor models. We believe, though, that much of this noise cancels out in Steps III and IV. 

In this context, it is important to note that some of us have been working with colleagues on an application of the approach presented here to food allocation~\cite{LKKK+18}. In this implementation\,---\,which is a collaboration with 412 Food Rescue, a Pittsburgh-based nonprofit\,---\,the set of alternatives includes hundreds of organizations (such as food pantries) that can receive incoming food donations. The voters in this implementation are a few dozen \emph{stakeholders}: representatives of donor and recipient organizations, volunteers (who deliver the donation from the donor to the recipient), and employees of 412 Food Rescue. These voters are obviously well informed, and the results of \citet{LKKK+18} indicate that they have been exceptionally thoughtful in providing their answers to pairwise comparisons.

Second, our dataset contains roughly 14 pairwise comparisons per voter on average. As suggested by Figure~\ref{fig:first}, this may not be sufficient for learning truly accurate voter models. However, subsequent work by \citet{KKAA+18} indicates that this problem can be alleviated by assuming that the parameters that determine the preferences of individual voters are drawn from a common distribution. This correlates the individual voter models, and, intuitively, allows the millions of examples to contribute to learning each and every model. Results based on the Moral Machine dataset indeed show that this technique leads to increased accuracy in predicting pairwise comparisons. In addition, in the work of \citet{LKKK+18}, many voters answered as many as 100 pairwise comparison queries, leading to strikingly accurate voter models that predict pairwise comparisons with roughly 90\% accuracy. 

Third, the choice of features in the Moral Machine dataset may be contentious. On the one hand, should we really take into account things like gender and profession to determine who lives and who dies? On the other hand, the set of alternatives is too coarse, in that it does not include information about probabilities and degrees of harm. As discussed by \citet{CSSD+17}, feature selection is likely to be a major issue for any machine-learning-based approach to ethical decision making. 
%Interestingly, though, it is again a nonissue in the food allocation domain of \citet{LKKK+18}: Extensive interviews with stakeholders have shown that only seven features (including, e.g., the distance between the donor and the recipient, and poverty level in the area of the recipient organization) affect decision making~\cite{LKL17}. 

\subsection{Extensions}

Going forward, most important is the (primarily conceptual) challenge of extending our framework to incorporate ethical or legal principles\,---\,at least for simpler settings where they might be easier to specify. The significant advantage of having our approach in place is that these principles do not need to always lead to a decision, as we can fall back on the societal choice. This allows for a modular design where principles are incorporated over time, without compromising the ability to make a decision in every situation. 

In addition, as mentioned in Section~\ref{sec:instantiation}, we have made some specific choices to instantiate our approach. We discuss two of the most consequential choices. First, we assume that the mode utilities have a linear structure. This means that, under the TM model, the estimation of the maximum likelihood parameters is a convex program (see Section~\ref{sec:instantiation}), hence we can learn the preferences of millions of voters, as in the Moral Machine dataset. Moreover, a straightforward summarization method works well. However, dealing with a richer representation for utilities would require new methods for both learning and summarization (Steps II and III).

Second, the instantiation given in Section~\ref{sec:instantiation} summarizes the $N$ individual TM models as a single TM model. While the empirical results of Section~\ref{sec:evaluation} suggest that this method is quite accurate, even higher accuracy can potentially be achieved by summarizing the $N$ models as a \emph{mixture} of $K$ models, for a relatively small $K$. This leads to two technical challenges: What is a good algorithm for generating this mixture of, say, TM models? And, since the framework of Section~\ref{sec:aggregation} would not apply, how should such a mixture be aggregated\,---\,does the (apparently mild) increase in accuracy come at great cost to computational efficiency?

\section*{Acknowledgments}

This work was partially supported by NSF grants IIS-1350598, IIS-1714140, IIS-1149803, CCF-1525932, and CCF-1733556; by ONR grants N00014-16-1-3075 and N00014-17-1-2428; by two Sloan Research Fellowships and a Guggenheim Fellowship; and by the Ethics \& Governance of AI Fund. 

%\newpage

%\newpage
\appendix

\section{Robustness of the Empirical Results}
\label{app:experiments}

In Section~\ref{subsec:synth}, we presented experiments using synthetic data, with the following parameters: each instance has $5$ alternatives, the number of features is $d=10$, and, in Step II, we let number of voters be $N=20$. In this appendix, to demonstrate the robustness of both steps, we show experimental results for different values of these parameters (keeping everything else fixed).

\begin{figure}[b]
  \centering
  \includegraphics[width=0.5\textwidth]{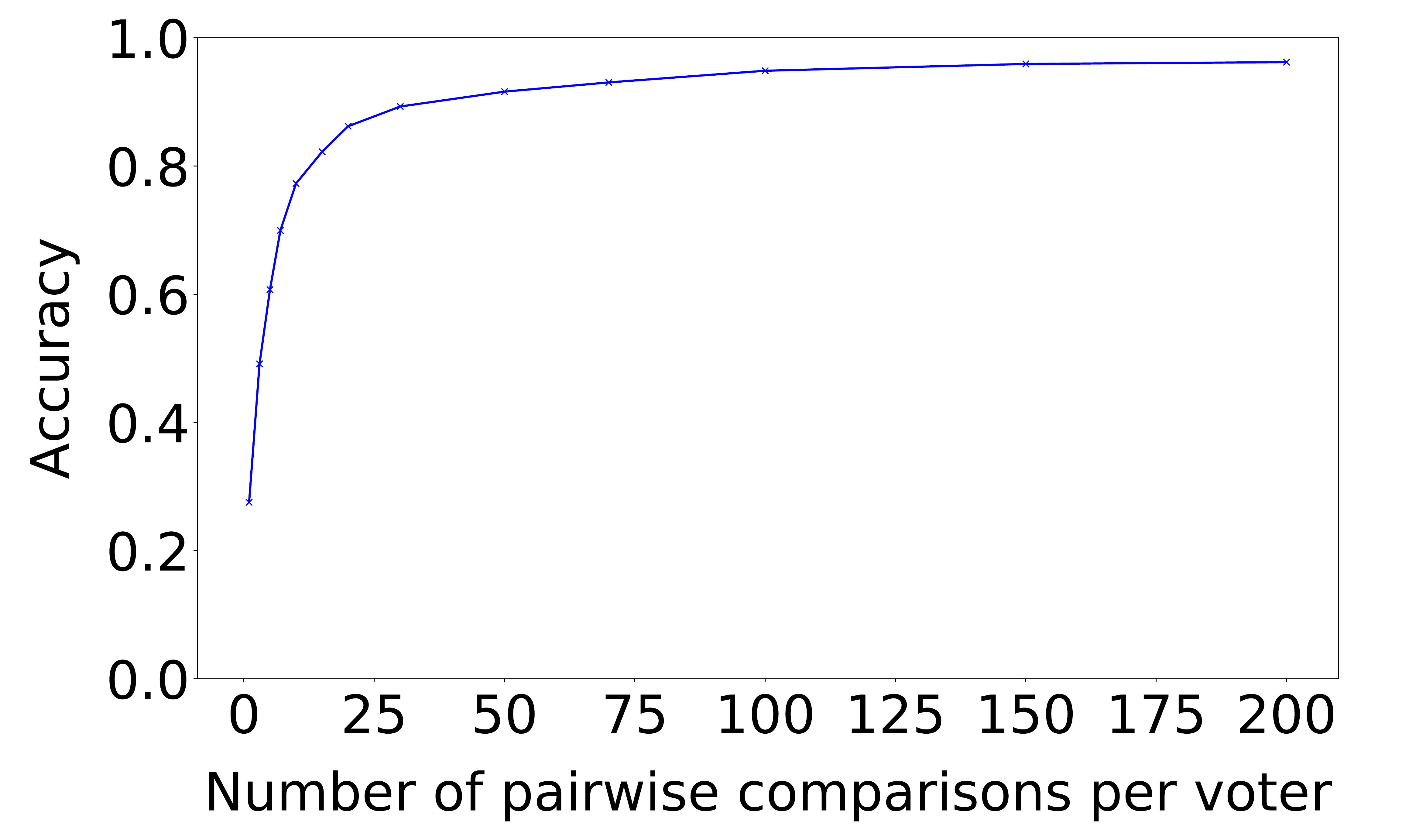}
  \caption{Accuracy of Step II with number of voters $N = 40$ (synthetic data)}
  \label{fig:40voters}
\end{figure}

\subsection{Number of Voters in Step II}

To show robustness with respect to the number of voters $N$ in Step II, we run the Step II experiments with $40$ (instead of $N=20$). The results are shown in Figure~\ref{fig:40voters}.

As before, we observe that the accuracy quickly increases as the number of pairwise comparisons increases, and with just $30$ pairwise comparisons we achieve an accuracy of $89.3\%$. With $100$ pairwise comparisons, the accuracy is $94.9\%$.

\begin{figure}[t]
\centering
\begin{subfigure}{.5\textwidth}
  \centering
  \includegraphics[width=\textwidth]{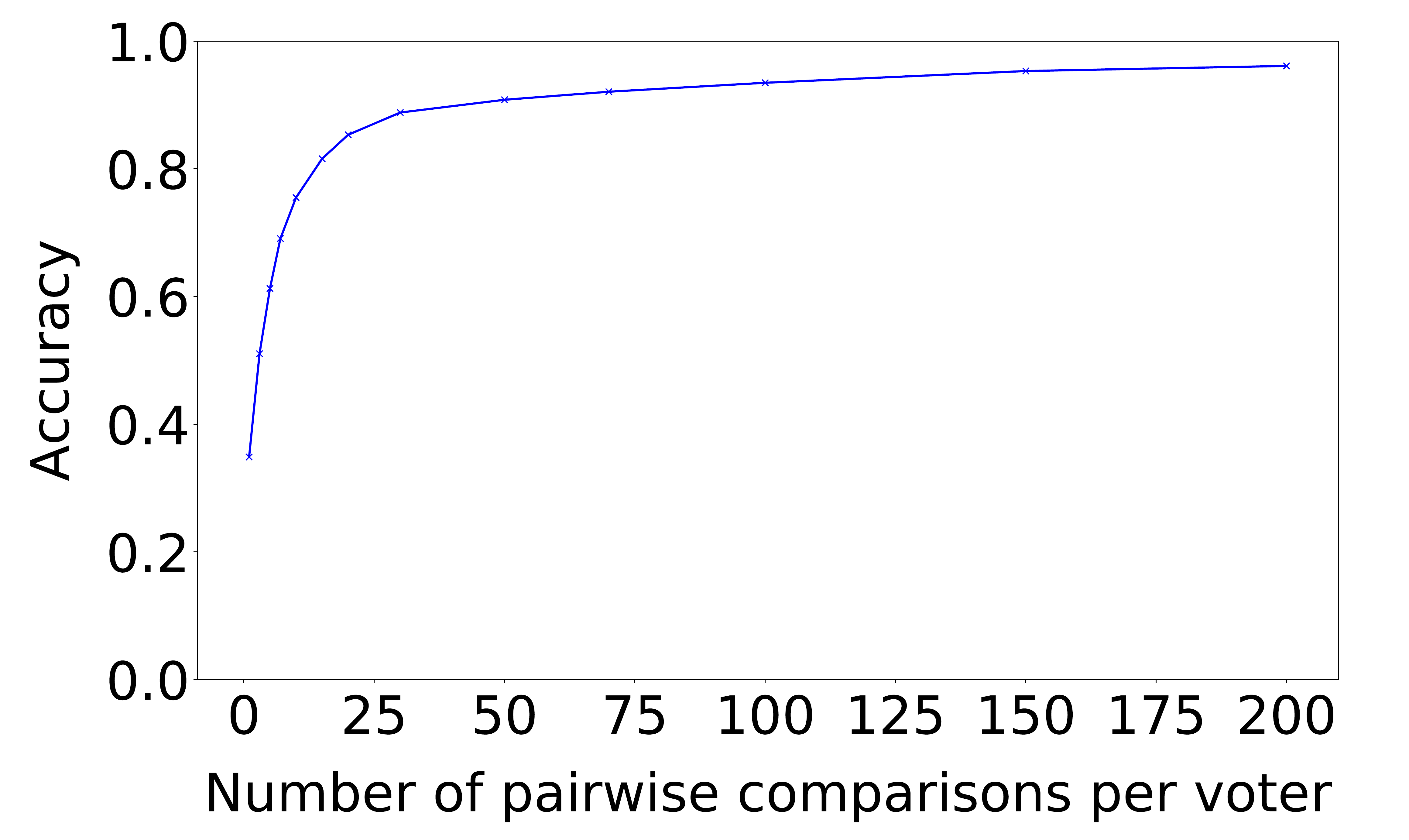}
  \caption{Accuracy of Step II}
  %\label{fig:sub1}
\end{subfigure}%
\begin{subfigure}{.5\textwidth}
  \centering
  \includegraphics[width=\textwidth]{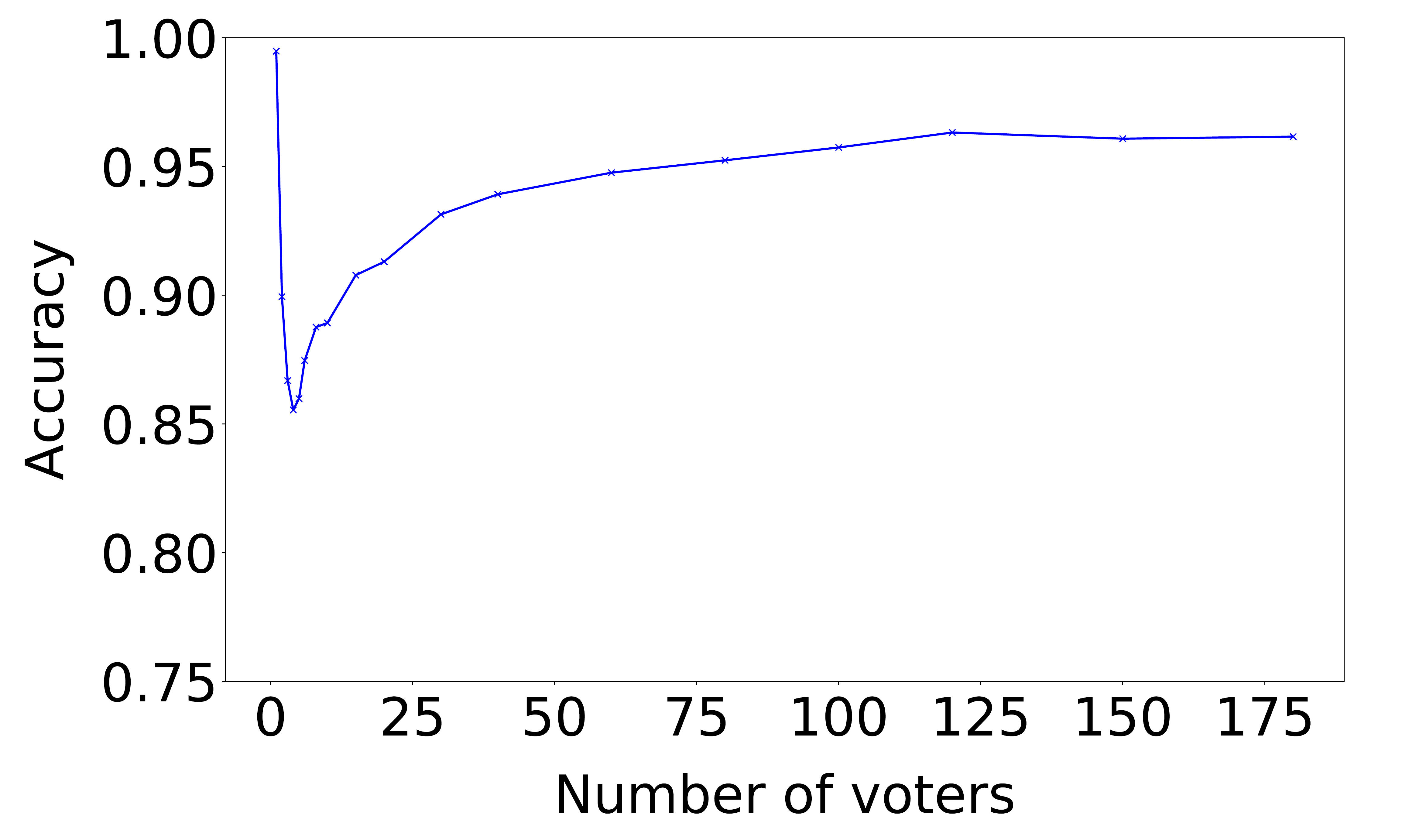}
  \caption{Accuracy of Step III}
  %\label{fig:sub2}
\end{subfigure}
\caption{Results with $3$ alternatives per instance (synthetic data)}
\label{fig:3alts}
\end{figure}

\begin{figure}[t]
\centering
\begin{subfigure}{.5\textwidth}
  \centering
  \includegraphics[width=\textwidth]{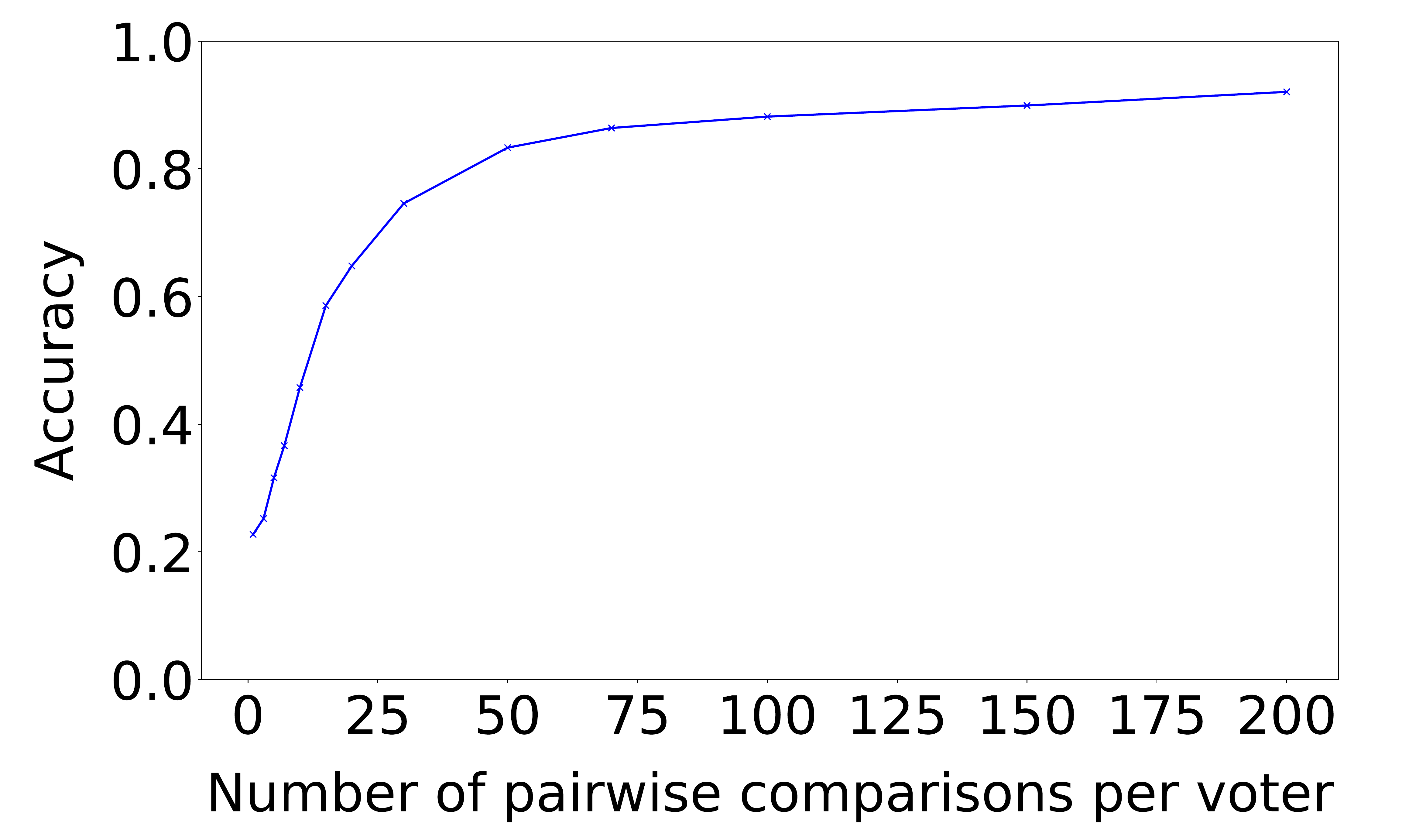}
  \caption{Accuracy of Step II}
  %\label{fig:sub1}
\end{subfigure}%
\begin{subfigure}{.5\textwidth}
  \centering
  \includegraphics[width=\textwidth]{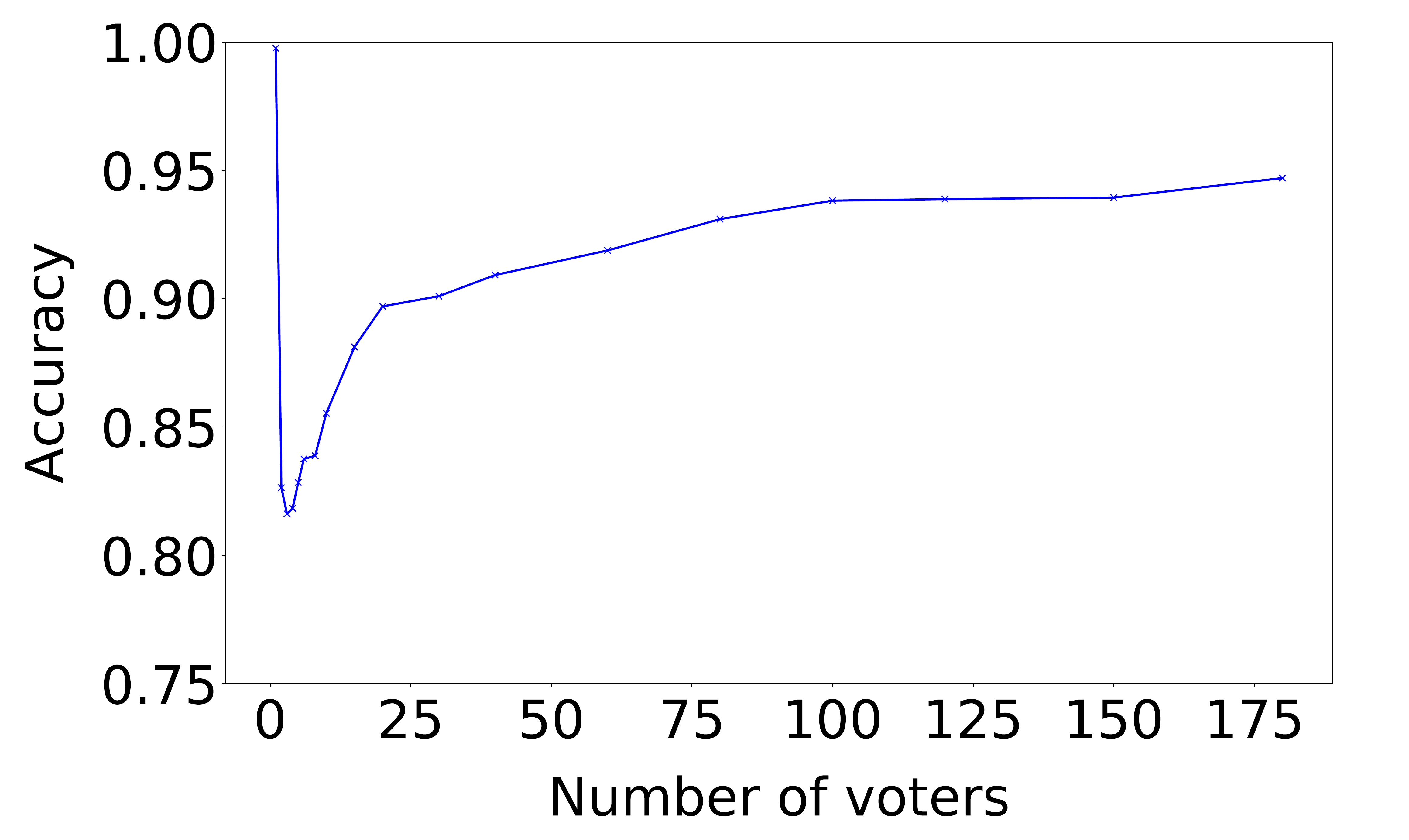}
  \caption{Accuracy of Step III}
  %\label{fig:sub2}
\end{subfigure}
\caption{Results with number of features $d = 20$ (synthetic data)}
\label{fig:20features}
\end{figure}

\subsection{Number of Alternatives}

To show robustness with respect to the number of alternatives, we run experiments with $|A|=3$ (instead of $|A|=5$). The results are shown in Figure~\ref{fig:3alts}.

Similarly to Section~\ref{subsec:synth}, for Step II, we observe that the accuracy quickly increases as the number of pairwise comparisons increases, and with just $30$ pairwise comparisons we achieve an accuracy of $88.8\%$. With $100$ pairwise comparisons, the accuracy is $93.5\%$. For Step III, we observe that the accuracy increases to $96.2\%$ as the number of voters increases.

\subsection{Number of Features}

To show robustness with respect to the number of features $d$, we run experiments with $d=20$ (instead of $d=10$). The results are shown in Figure~\ref{fig:20features}.

Again, for Step II, we observe that the accuracy quickly increases (though slower than in Section~\ref{subsec:synth}, because of higher dimension) as the number of pairwise comparisons increases. With just $30$ pairwise comparisons we achieve an accuracy of $74.6\%$, and with $100$ pairwise comparisons, the accuracy is $88.2\%$. For Step III, we observe that the accuracy increases to $94.7\%$ as the number of voters increases.

\end{document}